\documentclass[a4paper]{article}
\usepackage{algorithmic,algorithm}
\usepackage{amssymb, amsmath, amsfonts, amsthm}
\usepackage{dsfont, enumerate, gensymb}
\usepackage{url}
\usepackage{times}
\usepackage{graphics, graphicx, subfigure}
\usepackage[T1]{fontenc}
\usepackage[frenchb,english]{babel}
\usepackage[a4paper]{geometry}

\renewcommand{\mathbb}[1]{\mathds{#1}}
\newcommand{\mathbbm}[1]{\mathds{#1}}

\newtheorem{theorem}{{Theorem}} 
\newtheorem{lemma}{{Lemma}} 
\newtheorem{prop}{{Proposition}}

\newtheorem{assumption}{{Assumption}}
\newtheorem{remark}{{Remark}}

\newcommand{\EE}[1]{\mathbb{E}\left[{#1}\right]}
\newcommand{\PP}[1]{\mathbb{P}\left[{#1}\right]}
\newcommand{\DP}[2]{\left\langle #1, #2 \right\rangle}
\newcommand{\diag}[1]{\mbox{diag}(#1)}

\begin{document}
\title{Distributed User Profiling via Spectral Methods}
\author{
\parbox{.4\textwidth}
{
\begin{center}
Dan-Cristian Tomozei\\
EPFL, Switzerland \\
{dan-cristian.tomozei@epfl.ch}
\end{center}
}
\parbox{.4\textwidth}
{
\begin{center}
Laurent Massouli\'e\\
INRIA, France \\
{laurent.massoulie@inria.fr}
\end{center}
}
}

\date{}
\maketitle
\begin{abstract}
User profiling is a useful primitive for constructing personalised services, such as content recommendation. In the present paper we investigate the feasibility of user profiling in a distributed setting, with no central authority and only local information exchanges between users. We compute a {\em profile vector} for each user (i.e., a low-dimensional vector that characterises her taste) via spectral transformation of observed user-produced ratings for items. Our two main contributions follow:
\begin{enumerate}[i)]
\item \label{it1} We consider a low-rank probabilistic model of user taste. More specifically, we consider that users and items are partitioned in a constant number of classes, such that users and items within the same class are statistically identical. We prove that without prior knowledge of the compositions of the classes, based solely on few random observed ratings (namely $O(N\log N)$ such ratings for $N$ users), we can predict user preference with high probability for unrated items by running a local vote among users with similar profile vectors. In addition, we provide empirical evaluations characterising the way in which spectral profiling performance depends on the dimension of the profile space. Such evaluations are performed on a data set of real user ratings provided by Netflix.
\item We develop distributed algorithms which provably achieve an embedding of users into a low-dimensional space, based on spectral transformation. These involve simple message passing among users, and provably converge to the desired embedding. Our method essentially relies on a novel combination of gossiping and the algorithm proposed by Oja and Karhunen. 
\end{enumerate}

\end{abstract}

\thispagestyle{empty}

\pagebreak

\section{Introduction}
\label{sec:intro}
Recommendation systems have attracted much interest lately, mostly because of their relevance to core businesses of several major companies (e.g. Amazon, Netflix, Yahoo) who  offer large catalogues of products to a vast user base. While the advertisement of highly popular items is straightforward, a significant portion of business stems from sales of only mildly popular items. The latter cannot be advertised indiscriminately, and must be recommended to the ``right'' users, through targeted recommendations. Such companies dispose of large storage and computational resources which enable a centralised computation of recommendations. 

In this paper we take a different perspective on the problem of recommendation. Namely, we aim to develop strategies suited to {\em distributed} operation, where the burden of recommendation is not offloaded to the server, but is rather shared among the users. More specifically, we propose the following two-stage approach for generating recommendations:
\begin{itemize}
\item In the first stage, distributed algorithms assign coordinates (or {\em profiles}) to the users within a certain {\em profile space}, such that proximity in this space translates to proximity of user taste for content. We say that such algorithms perform {\em user profiling}.
\item In the second stage, recommendations are obtained via simple and distributed algorithms which rely on the primitive of user profiling. We thereby avoid the need for complex machine learning techniques.
\end{itemize}
The performance of such an approach will depend heavily on the properties of the considered embedding of users in the profile space. For this reason, the focus in this paper is on the first stage of the process, i.e., {\em user profiling}. Namely, we argue that {\em spectral} profiling techniques retrieve hidden structure.

The techniques employed in a centralised setting for generating content recommendation are widely known under the generic name of ``collaborative filtering''. They are typically implemented by a provider who wishes to offer a recommendation service to a large customer base. In such a setting, the information requested from the customers (or users) is typically related both to their identity (via the registration procedure) and to their taste (via the opinions they express regarding the items).

It is not clear to which extent identity information characterises user taste. Moreover, the nature of such information gives rise to privacy concerns. On the contrary, the opinions that users express about items constitute the truly relevant data for solving the problem. For this reason, we advocate a purely agnostic approach to recommending content, which does not use information about the real identity of users, or the nature of content. 

Opinions are expressed in the form of ratings assigned by a user to the items she has already purchased. Ratings characterise the satisfaction of a user with respect to a specific item. They are discrete and range from a lowest to a highest value (e.g., number of stars). In particular, the mere fact that a user has consumed or not a specific item can be regarded as a {\em binary rating}. In this paper we consider the latter form of rating. 

Since the number of items on offer from the provider is overwhelmingly large, the vast majority of users only consume a small fraction of items. Hence, for a typical user, only a small number of ratings are known. In the case of binary ratings, if an item has not been consumed it does not necessarily follow that the user dislikes it. It is possible that the user is simply unaware of the item's existence. Hence, in this case we cannot distinguish between missing ratings and disliked items.

A recent illustration of the possible machine learning techniques and
of the corresponding performance comes from the Netflix prize
competition~\cite{nfp}. The goal of the competition was to design an
algorithm that, when trained on a data set made publicly available by
Netflix, would manage to improve prediction accuracy by $10\%$
(measured via Root Mean Squared Error) compared to the proprietary
Cinematch algorithm. The designers of such an
algorithm~\cite{netflix-win} were awarded a prize of \$1M three years
after the start of the competition. It is important to note that the
last two years of the three were spent trying to improve the gain in
prediction from $8.42\%$ to $10.04\%$.

The strenuous advancement of the Netflix prize suggests the existence
of an important obstacle in the way of achieving high prediction
accuracy. A possible explanation is given in the study conducted by
Amatriain et al.~\cite{natnoise}, which showed that when presented
with a movie title several times, users provide inconsistent
ratings. Hence, there exists an implicit noise in any collection of
ratings due to the fact that human taste is variable and not easily
quantified. The authors~\cite{natnoise} used the RMSE to characterise
the distance between two sets of ratings assigned by the same users to
the same movies. They found RMSE values between $0.55$ and $0.63$. It
is arguable whether these specific values can be compared to the
target RMSE value $0.8563$ of Netflix (the considered movies in the
conducted study~\cite{natnoise} are indeed a subset of the ones in the
Netflix data set, but their number is significantly smaller, the user
base is different and also much smaller, etc.). However what they
suggest is that there exists a RMSE threshold (of the order of the
Netflix target RMSE) that cannot be overcome by any recommendation
algorithm, seeing as users themselves do not necessarily provide
consistent ratings. Since user preference has a significant random
component, parameters of a probabilistic model are the best suited to
characterize user taste. Throughout this paper we make the following
\begin{assumption}
\label{ass:proba}
The taste of each user is characterised by a certain probability
distribution defined on the set of all possible ratings for the set of
items (which includes the possibility that an item is ``not
rated''). A user's observed ratings are obtained by sampling her
corresponding multi-dimensional distribution.
\end{assumption}

We now restrict ourselves to the binary rating model, where we do not
distingush between an item that was not rated and a ``disliked'' item
(i.e., $0$-rated items). We describe a natural way of representing the
observed purchases. Denote the set of users by $\mathcal U$ and the
set of items by $\mathcal F$.
Consider a rectangular matrix $S\in \{0,1\}^{|\mathcal U|\times
  |\mathcal F|}$ which we call {\em the rating matrix} (in the
literature it is also referred to as the ``term-document''
matrix). Each row corresponds to a user and each column corresponds to
an item. An entry $S_{ui}$ corresponds to the user-item pair $(u,i)\in
\mathcal U\times \mathcal F$. The entry holds $1$ if user $u$ has
purchased item $i$, and $0$ otherwise. As previously stated, most of
the zero entries in the rating matrix correspond to cases in which a
user has not considered purchasing a specific item.

By Assumption~\ref{ass:proba}, each row of the rating matrix is a
realisation of a $\{0,1\}$ random $|\mathcal F|$-dimensional
vector. Relying only on matrix $S$, we need to assign profiles to the
users, such that users with similar taste have similar profiles.

The instances of the problem we consider are extremely large, it is
not uncommon to have a user base $\mathcal U$ of the order of millions
and a catalogue of items $\mathcal F$ of the order of tens of
thousands. In the case of binary ratings, simply representing the
probability distribution in Assumption~\ref{ass:proba} for a single
user requires an exponential amount of memory $2^{|\mathcal F|}$. We
are thus constrained to consider very simple approximations of such
probability laws, for the sake of computational tractability.

Like most proposed models of user taste found in the literature
(e.g.,~\cite{montanari}), we consider a low-rank model of user taste.
We stress that our model is {\em probabilistic}, as opposed to the
deterministic one of~\cite{montanari}. More specifically, we make the
following
\begin{assumption}
\label{ass:lowr}
Each entry $S_{ui}$ of the rating matrix is given by an independent
Bernoulli random variable of parameter $\bar S_{ui}$. The matrix $\bar
S = (\bar S_{ui})$ has rank $K$, where $K\ll |\mathcal U|$ and $K\ll
|\mathcal F|$.
\end{assumption}

In Section~\ref{sec:theory} we propose a user profiling technique
based on the Singular Value Decomposition of matrix $S$. For a
probabilistic model of user taste satisfying
Assumption~\ref{ass:lowr}, and under further weak statistical
assumptions, we prove that a simple voting scheme among users with
similar profiles manages to produce accurate recommendations for most
of the items with high probability. Furthermore, we use actual movie
ratings to compute the profiles of anonymous users of the Netflix
system. We find that users with similar profiles have indeed similar
taste in movies.

Motivated by the ability to recover hidden structure of the spectral techniques,
 in the second part of this paper (Section~\ref{sec:drw}), we design a distributed algorithm that computes individual spectral profiles based on local exchanges among users. 
We prove almost sure convergence of the algorithm and provide evaluations on a synthetic trace. We conclude in Section~\ref{sec:concl}.

\section{Related Work}
\label{sec:rwork}


Keshavan et al.~\cite{montanari} consider the problem of low rank matrix completion. They show that for a constant rank $r = O(1)$ ``well-behaved'' matrix, it is sufficient to have $\Omega(N \log N)$ revealed entries in order to be able to achieve exact matrix reconstruction. For a square matrix, this corresponds to an average degree of $\Omega(\log N)$, just as in our result. The drawback to a direct application of this result to the user taste prediction problem is the fact that the sought matrix is deterministic. Implicitly, it would mean that user ratings are deterministic, and that rating matrices are low rank. In contrast, the models we consider are probabilistic.

In Section~\ref{sec:theory} we consider a low-rank probabilistic model of user taste. Users are partitioned into classes, such that users within the same class are statistically identical. We show (Theorems~\ref{thm:main} and~\ref{thm:rmain}) that the profile vectors corresponding to each user computed via spectral methods are clustered around distinct points corresponding to the classes. In this respect, our results are related to spectral clustering. There is a vast literature on the topic of spectral clustering, of which we now give a brief overview.

Among the most relevant work,  Ng et al.~\cite{jordan}  propose a clustering algorithm based on spectral decomposition. Our results provide a comprehensive analysis by giving conditions under which the underlying partition into classes is retrieved exactly.

In~\cite{hopcroft}, Dasgupta et al. propose an algorithm based on
iterative splitting of groups into two subgroups. In contrast, we
obtain the desired groups in one go. In~\cite{mcsherry}, McSherry
proposes a different clustering method based on projections onto the
column space of the original matrix. In both~\cite{hopcroft}
and~\cite{mcsherry}, the probabilistic model of the observed
similarity matrix is akin to ours (which extends the classical
``planted partition'' model). However, our model is less expressive in
the sense that we do not consider ``mixtures'' of the distributions
that characterise classes of users, and moreover in our case each user
class accounts for a constant fraction of the total number of
users. That said, the required separability conditions are similar,
and we establish our results under far less stringent conditions on
the average degree of the observed matrix. Namely, we require an
average degree of order $\Omega(\log(N))$ while they require an order
of $\Omega(\log(N)^6)$.

The recent paper~\cite{eigenspaces} by Shi et al. discusses rationales for choosing which eigenvectors to use when performing spectral clustering. This issue is to a large extent complementary to the ones we address in this paper. We could rely on~\cite{eigenspaces} to specify which eigenvectors to keep in our profiling context.

In Section~\ref{sec:rect} we extend the previous results in the more
general setting of content recommendation, where a sparse so-called
rating matrix (or ``term-document'' matrix) is available. The
literature in this field is extensive, for brevity we mention only a
few significant works:~\cite{ChaudhuriR08,Dasgupta05spectralclustering,KumarK10}.

In Section~\ref{sec:distro} we propose a method for computing the eigenvectors of the adjacency matrix of a graph in a distributed manner. A variant of the method was briefly described in~\cite{poster}. Eigenvector extraction is the object of Oja's algorithm~\cite{oja}. This basic algorithm was refined by Borkar and Meyn~\cite{bmoja}. None of these approaches is distributed however. Our contribution in Section~\ref{sec:distro} consists precisely in augmenting these methods to make them distributed.

A significant contribution towards computing the top $k$ eigenvectors of a symmetric weighted adjacency matrix in a distributed fashion was brought by Kempe and McSherry~\cite{kempe}. The setting is similar to the one we consider in Section~\ref{sec:distro}. The authors give bounds on the required running time of their algorithm. Due to the fact that we explicitly introduce noise in our iterations, obtaining such bounds is more difficult in our case. Both our algorithm and the distributed gossiping algorithm in~\cite{kempe} perform iterations to converge to the desired eigenvectors. In the latter, at each iteration all participating nodes first perform a coordinate-update step followed by an orthonormalisation step (that lasts for a determined number of rounds ensuring a bounded error). By explicitly separating the two steps in each iteration, synchronised time becomes a necessary assumption. In contrast, our approach facilitates an asynchronous implementation, since the coordinate-update and the orthonormalisation steps are performed simultaneously with different gains (and thus on different time scales). Furthermore, the algorithm we propose in Section~\ref{sec:distro} uses few elementary computations at each node, whereas in~\cite{kempe} at every iteration each node needs to perform a Cholesky factorisation.

A similar approach to the one presented in this paper has been taken in a recent publication~\cite{gpca}. The authors aim to determine the eigenvectors of a deterministic matrix based on random sparse observations. They derive useful bounds on convergence time. However, the gossiping stage in their proposed algorithm is treated as a ``black box''. We explicitly construct an algorithm that incorporates two stages: gossiping, performed on a faster time scale, and Oja's method, performed on a slower time scale. Moreover, we explicitly determine multiple eigenvectors, whereas the authors of ~\cite{gpca} focus on determining a single eigenvector, and argue that the extension can be achieved. Finally, we propose an asynchronous algorithm. We show on synthetic data that the latter determines the desired eigenvectors.

\section{Spectral Recovery of Probabilistic Taste}
\label{sec:theory}
We begin by analysing a simple setting in which the observations consist of measures of similarity between users. We prove that a profiling technique based on the spectral decomposition of the square matrix regrouping the observed similarities between pairs of users successfully recovers hidden structure.

We apply these findings to the case in which ratings of items by users are observed. We show that a simple distributed voting algorithm provides asymptotically accurate predictions for most items.

Finally, we observe the benefits of spectral profiling on a real trace.

We make use of notations described in Table~\ref{tab:not}. Unless otherwise indicated, all vectors are column vectors.

\begin{table}[ht]
\begin{center}
\begin{tabular}{c|p{.6\textwidth}}
$A'$ & The transposition of matrix $A$ \\
$x'$ & The transposition of column vector $x$ \\
$e$ & The all-ones column vector \\
$\diag \alpha$ & The $K\times K$ diagonal matrix having the elements on the main diagonal given by the $K$-dimensional vector $\alpha$ \\
$\|y\|_\alpha$ & The $\alpha$-norm of vector $y$, where $(0 < \alpha_k < 1)_{k=1}^K$ and $y$ are column vectors with the same dimension, $\|y\|_\alpha := \sqrt {\sum_k \alpha_k y_k^2}$ \\
$y(k)$, $y_k$ & The $k$-th element of column vector $y$
\end{tabular}
\end{center}
\caption{General notations}
\label{tab:not}
\end{table}

\subsection{Similarity-based Profiling}
\label{sec:sq}
Denote by $N=|\mathcal U|$ the number of users. Let us consider in this first stage that we are given partial observations of user taste similarity in the form of a symmetric matrix 
$$A \in \{0,1\}^{N \times N}.$$
Namely, for any two users $u, v$ the elements $A_{uv} = A_{vu}$ take value $1$ if users $u$ and $v$ have been evaluated as similar, and value $0$ if the users are deemed dissimilar, or if the similarity between the two users has not been evaluated. By convention $A_{uu} = 0$.

We propose the following {\em spectral} representation of users based on these partial similarity observations. For some fixed dimension $L$, extract the \emph{normalised} eigenvectors $x_1,\ldots,x_L$ corresponding to the $L$ largest magnitude eigenvalues of matrix $A$. We define the {\em profile space} as ${\mathbb{R}}^L$. In the profile space, to each user $u$ there corresponds a scaled row vector $\sqrt N z_u'$, where
$$
z_u'=(x_1(u),\ldots,x_L(u)).
$$
We refer to this vector as the {\em profile of user $u$}. The scaling factor $\sqrt N$ is introduced to compensate for the fact that the eigenvectors of the $N \times N$ matrix $A$ are taken of norm $1$. In what follows, we propose a simple probabilistic user taste model for which this spectral representation of users enables us to retrieve hidden structure.

\subsubsection{Statistical Model}
\label{sec:sqmod}
We assume the following probabilistic model of user taste: The $N$ users are partitioned into $K$ classes $C_1, \dots, C_K$, such that users within the same class are statistically identical. We denote the size of class $C_k$, $1\le k\le K$, by $|C_k| = \alpha_kN$, for fixed $\alpha_k > 0$, which are such that $\sum_k \alpha_k = 1$. For any user $u \in \mathcal U$ we denote by $k(u)$ her unique corresponding class.

Each pair of classes $1\le k,\ell \le K$ is characterised by probabilities $b_{k\ell} = b_{\ell k}$ as follows. User $u\in C_k$ and user $v \in C_\ell$ are similar with probability $b_{k\ell}$ and dissimilar with probability $1-b_{k\ell}$. Moreover, for all pairs of users, their similarity is observed with probability $p = \frac \omega N$, where $\omega$ is a parameter of the model. The similarity of unobserved pairs of users is set to $0$ by default. 

Equivalently, for all ordered pairs of users $u < v$, the observed similarities $A_{uv} = A_{vu}$ are the outcome of independent Bernoulli random variables of parameters $pb_{k(u)k(v)}$. As previously stated, the diagonal elements are all null, $A_{uu} = 0$.

Given the observed similarity matrix $A$, without knowledge of the $K\times K$ profile similarity matrix $B = (b_{k\ell})_{k,\ell}$, we wish to recover the partition of users in the unknown classes $C_k$ using their spectral representation. Below we provide sufficient scaling assumptions for which such recovery is possible.

In Table~\ref{tab:sqndx} we summarise the notations we have introduced so far:

\begin{table}[h]
\begin{center}
\begin{tabular}{c|p{.6\textwidth}}
$u,v$ & Indices referring to users \\
$k, \ell$ & User profile indices \\
$k(u)$ & Index of the profile containing user $u$ : $u \in C_{k(u)}$ \\
$\{C_k\}_{k=1}^K$ & Partition of users into $K$ unknown disjoint profiles \\
$\alpha$ & Vector grouping fractions of users per profile: $|C_k|=\alpha_k N$ \\
$B$ & Constant unknown $K\times K$ matrix of probabilities \\
$p = \frac \omega N$ & Probing probability \\
$L$ & Dimension of profile space
\end{tabular}
\end{center}
\caption{Notations and conventions}
\label{tab:sqndx}
\end{table}

\subsubsection{Scaling Assumptions}
Let us describe the dependence of the various parameters of the model on the number of users.

We have made the Assumption~\ref{ass:lowr}, namely that our model is low-rank. Thus, we consider that the number of classes $K$, as well as the fraction $\alpha_k$ of users in each class, and the similarity probability matrix $B = (b_{k\ell})_{k,\ell}$ are constant as the number of users in the system grows.

However, we assume that the probing probability $p$ vanishes as the number of users grows. Namely, $\omega$ goes to infinity slower than $N$, 
\begin{equation}
\label{eq:omega}
\omega \to_N \infty, \ \omega = o(N).
\end{equation}
Hence, each user probes on average a fraction $p = \frac \omega N \to 0$ of users with whom she evaluates taste similarity. Equivalently, matrix $A$ can be regarded as the adjacency matrix of a random graph on the set of users having average degree of $\Theta(\omega)$.

\subsubsection{Hidden Structure Recovery}
Theorem~\ref{thm:main} below states that under mild conditions, for large $N$, the spectral representations $\sqrt N z_u'$ of users are clustered according to their respective classes $C_k$. 
Consider the vector
$$\alpha := (\alpha_{k})_{k=1}^{K}.$$
The $\alpha$-norm of a $K$-dimensional vector $t$ is $\|t\|^2_\alpha = \sum_{k} \alpha_{k} t_{k}^2$ (see Table~\ref{tab:not}). Define the following constant matrix
$$
M:=(b_{k\ell}\alpha_{\ell})_{1\le k,\ell\le K}.
$$
Before stating the theorem, we introduce the following conditions:

\begin{subequations}
\label{C14}
\begin{align}
\label{C1}
\parbox[b]{.85\textwidth}{The dimension of the profile space is upper bounded by $L \le \mbox{rank}(M)$.}\\
\label{C2}
\parbox[b]{.85\textwidth}{The $L$ largest magnitude eigenvalues of $M$ have distinct absolute values.}\\
\label{C3}
\parbox[b]{.85\textwidth}{The corresponding eigenvectors normalised under the $\alpha$-norm, $y_1,\ldots,y_{L}$, satisfy: 
$$
t_k'\ne t_\ell', 1\le k\ne \ell\le K,
$$
where $t_k':=(y_1(k),\ldots,y_L(k))$.} \\
\label{C4}
\parbox[b]{.85\textwidth}{$\omega \ge C\log(N)$, for some absolute constant $C$.}
\end{align}
\end{subequations}

\begin{theorem}
\label{thm:main}
Under assumptions \eqref{C14}, with probability $1-o(1)$, a fraction of $1-o(1)$ users $u$ is such that $||\sqrt{N}z_u' -t_{k(u)}'||=o(1)$, where $k(u)$ denotes the class of $u$. That is to say, most users have their (scaled) profile vector close to a fixed vector characteristic of their class.
\end{theorem}

Each vector $t'_k$ corresponds to a class $C_k$. If two classes $k,\ell$ were such that $t_k = t_\ell=t$, the theorem ensures that the profiles of users of both classes would be grouped around the constant vector $t$ in the profile space. Hence, it would be impossible to distinguish between the users of the two classes based solely on their profile vectors. Condition \eqref{C3} ensures that
$$
\|t_k'- t_\ell'\| = \Omega(1),
$$
thus guaranteeing the ability to distinguish between distinct classes for a large enough number of users.

We prove that eigenvectors corresponding to non-zero eigenvalues of matrix $A$ can be used to recover the hidden classes. Condition \eqref{C1} ensures selection of a suitable dimension $L$. We impose the technical condition \eqref{C2} for presentation ease.

Condition \eqref{C4} gives a lower bound of $\log N$ on order of the required average user neighbourhood size $\omega$. The theorem states that for $\omega$ growing at least as fast as $\log N$, the initial partition in profiles $C_k$ can be recovered with high probability for almost all users. 

\begin{remark} 
Concerning condition \eqref{C4}, it is plausible that an even lower requirement for $\omega$ (i.e. constant) suffices. Such a case has been explored in the context of bounding the second eigenvalue of the adjacency matrix of a sparse random graph to $O(\sqrt \omega)$, by removing high degree outlier nodes from the graph (see~\cite{feige,Coja-Oghlan}).
\end{remark}

\begin{remark}Other flavours of matrices than the adjacency matrix $A$ could be considered for spectral analysis (e.g. Laplacian matrix, normalised Laplacian matrix). We do not address either of these scenarios in the present work.
\end{remark}

We now give the main steps in the proof of Theorem~\ref{thm:main}. The auxiliary lemmas are proved in the appendix. Consider the matrix
$$\bar A := (pb_{k(u)k(v)})_{u,v},$$
which, according to our model, is equal to the expectation $\mathds E A$ of the partially observed similarity matrix $A$. We can write $A = \bar{A} + Q$, where $Q:=A-\bar{A}$.

The theorem relies on the fact that the block matrix $\bar A$ imposes the eigenvalues and eigenvectors of $A$, while the perturbation matrix $Q$ has little influence therein, as follows from the lemmas below:
\begin{lemma}
\label{lemma:dominant}
The top $L$ largest magnitude eigenvalues of $\bar A$ have distinct absolute values and are order of $\Theta( \omega )$. The normalised eigenvectors $(\bar x_\ell)_{\ell = 1}^L$ corresponding to these eigenvalues are constant on indices corresponding to each user class. Specifically, using the $y_\ell$ defined in \eqref{C3}, we can write
\begin{equation}
\label{eq:eigcorresp}
\bar x_\ell(u) = \frac {y_\ell(k(u))} {\sqrt N}, \ \forall u \in \mathcal U, \ 1\le \ell \le L.
\end{equation}
\end{lemma}

Use the following ordering of the eigenvalues of $\bar{A}$ and $A$:
\begin{eqnarray}
\label{eq:eigs}
|\bar \lambda_1| > |\bar \lambda_2| > ... > |\bar \lambda_L|  \\
\nonumber
|\lambda_1| \ge |\lambda_2| \ge  ... \ge  |\lambda_L|.
\end{eqnarray}
We also denote by $\bar x_k$ and $x_k$ the corresponding \emph{normalised} eigenvectors.

To control the influence of the perturbation matrix $Q$ we use the following
\begin{lemma}
\label{lemma:feige}
Consider a square $N\times N$ symmetric $0$-diagonal random matrix $A$
such that its elements $A_{ij} = A_{ji}$ are independent Bernoulli
random variables with parameters $\mathbbm E A_{ij} = p_{ij} = a_{ij}
\omega N^{-1}$, where the $a_{ij}$ are constant and $\omega = \Omega(
\log N )$. Then with high probability the spectral radius of the
matrix $A - \mathbbm E A$ satisfies the upper bound $\rho( A -
\mathbbm E A ) \le O( \sqrt \omega )$.
\end{lemma}
We provide a proof of this lemma in Appendix~\ref{sec:proof-feige}
which relies on the work of Feige and Ofek~\cite{feige}.

Denote by $D = \diag{b_{k(u)k(u)} \frac \omega N, 1 \le u \le N}$. By application of Lemma~\ref{lemma:feige} to matrix $A$, we get that the spectral radius of $Q_0 := A - \mathbbm E A$ is upper bounded by $O( \sqrt \omega)$ with high probability. Since $\rho(D) \le O( \frac \omega N )$, we have that the spectral radius of $Q = Q_0 - D$ is also upper bounded by $O( \sqrt \omega)$ with high probability.

The previous two lemmas are instrumental in proving that $\bar A$ and $A$ have the same spectral structure:
\begin{lemma}
\label{lemma:pert}
Using the ordering~(\ref{eq:eigs}) for the eigenvalues of $\bar A$ and $A$, it holds that for all $1 \le k \le K$
\begin{eqnarray}
\label{eq:eigdiff}
\left| \,|\lambda_k| - |\bar \lambda_k| \,\right| \le O(\sqrt \omega) & & \mbox{whp}, \\
\label{eq:vecdiff}
\sin( \widehat{ x_k, \bar x_k } ) \le O( \omega ^ {-1/4} ) & & \mbox{whp}.
\end{eqnarray}
\end{lemma}

We conclude by an averaging argument.
Lemma~\ref{lemma:pert} shows that with high probability 
we have $\|x_\ell - \bar x_\ell\|^2 \le O(\omega^{-1/2})$, for all $1\le\ell\le L$. Condition~\eqref{C3} guarantees that $\|t_k' - t_l'\| = \Omega(1)$. 

The fraction of users that have their profile vectors at a distance
larger than some constant $a>0$ from the vector $t_k$ corresponding to
their class is
\begin{align*}
\frac 1 N \left|\left\{ u: \|z_u'-t_{k(u)}'\|\ge a\right\}\right|& \le\frac{1}{N}\sum_{u=1}^N\frac{\|\sqrt{N} z_u'-t_{k(u)}'\|^2}{a^2} 
\mathop{=}^{\eqref{eq:eigcorresp}} \sum_{\ell = 1} ^ L \frac {\|x_\ell - \bar x_\ell\|^2}{a^2} = O(a^{-2}\omega^{-1/2}).
\end{align*}

Thus we will be able to conclude the result of the theorem if we can find an $a$ such that $a=o(1)$ and $a^{-2}\omega^{-1/2}=o(1)$. It is easy to see that for instance $a=\omega^{-1/6}$ satisfies these conditions. \hfill$\Box$

\subsection{Application: Extension to Content Recommendation}
\label{sec:rect}
Let us now consider the scenario exposed in
Section~\ref{sec:intro}. Denote again by $N = |\mathcal U|$ the number
of users and by $F = |\mathcal F|$ the number of items. Assume without
loss of generality that $N>F$. We can write $F = \gamma N$, with $0 <
\gamma < 1$. We consider the rectangular observed rating matrix
$$S \in \{0,1\}^{N\times F},$$ where $S_{ui} = 1$ if user $u$ has
rated and liked item $i$ and $0$ otherwise. If an entry $S_{ui}$ is
null, it is not necessarily true that user $u$ dislikes item $i$ (the
item might have simply not been rated).

We propose the following representation of users based on this collected information. For some dimension $L$, extract the $L$ \emph{normalised left singular vectors} $x_1, \dots, x_L$ of $S$, corresponding to its $L$ largest singular values. Like in the previous subsection, consider the $L$-dimensional profile space $\mathbbm R^L$, in which we associate a scaled row vector $\sqrt N z_u'$ to user $u$, where
$$
z_u' = (x_1(u), \dots, x_L(u)).
$$
The scaling by $\sqrt N$ is again due to the fact that the singular vectors are normalised.

Recall that the Singular Value Decomposition (SVD) of a rectangular real matrix $S = X \Sigma Y$ is always well defined, where matrices $X \in \mathbb R^{N\times F}$ and $Y \in \mathbb R^{F\times F}$ are unitary and matrix $\Sigma \in \mathbb R^{F\times F}$ is positive diagonal.

Again, we claim that users with similar taste will be mapped to close-by locations in the profile space. More specifically, we show that local voting in the profile space provides users with accurate recommendations for most of the items. In what follows, we give a probabilistic model of user taste and scaling assumptions for which we prove these claims.

\subsubsection{Statistical Model}
We consider a probabilistic model of user taste similar to the one in Section~\ref{sec:sqmod}. The $N$ users are partitioned into $K$ disjoint classes $C_1, \dots, C_{K}$ and the $F$ items are partitioned into $K'$ disjoint classes $D_1, \dots, D_{K'}$. Users and items in the same class are statistically identical. The size of user class $C_k$ is denoted by
$$
|C_k| = \alpha_k N, \quad 1 \le k \le K,
$$
while the size of item class $D_{k'}$ is denoted by
$$
|D_{k'}| = \beta_{k'} F,  \quad 1 \le k' \le K'.
$$
The $(\alpha_k)_k$ and $(\beta_{k'})_{k'}$ are strictly positive and sum to $1$:
$$
\sum_k \alpha_k = \sum_{k'} \beta_{k'} = 1, \quad \alpha_k > 0, \  \beta_{k'} > 0.
$$
For any user $u \in \mathcal U$ we denote by $k(u)$ her unique corresponding user class, and for any item $i \in \mathcal F$ we denote by $k'(i)$ its unique corresponding item class. 

Pairs of user and item classes $(C_k, D_{k'})$, $1\le k \le K, 1\le k' \le K'$, are characterised by probabilities $r_{kk'}$ in the following way: Any user $u \in C_k$ likes any item $i\in D_{k'}$ with probability $r_{kk'}$ and dislikes it with probability $1-r_{kk'}$. For all user-item pairs $(u \in \mathcal U, i \in \mathcal F)$, $u$ decides to rate $i$ with probability $p = \frac \omega N$. Equivalently, any element $S_{ui}$ of the observed rating matrix $S$ is obtained by drawing a Bernoulli random variable of parameter $p\cdot r_{k(u)k'(i)}$.

Given the observed rating matrix $S$, without knowledge of the $K\times K'$ affinity matrix $R := (r_{kk'})_{k,k'}$, we wish to recover the partition of users in the $K$ classes $C_k$ by making use of their spectral representation. We provide sufficient scaling conditions in what follows.

We summarise the notations in Table~\ref{tab:rtndx}.

\begin{table}
\begin{center}
\begin{tabular}{c|p{.6\textwidth}}
$N$ & Number of users \\
$F = \gamma N$ & Number of items \\
$k'(i)$ & Index of the class containing item $i$ : $i \in D_{k'(i)}$ \\
$\{C_{k}\}_{k=1}^{K}$ & Partition of users into $K$ disjoint classes \\
$\{D_{k'}\}_{k'=1}^{K'}$ & Partition of items into $K'$ disjoint classes \\
$\alpha$ & Vector grouping fractions of users per class: $|C_{k}| = \alpha_{k}N$ \\
$\beta$ & Vector grouping fractions of items per class: $|D_{k'}| = \beta_{k'}F$ \\
$R$ & Constant unknown $K\times K'$ affinity matrix
\end{tabular}
\end{center}
\caption{Notations}
\label{tab:rtndx}
\end{table}

\subsubsection{Scaling Assumptions}
We make again Assumption~\ref{ass:lowr} and take the number of classes
of users $K$ and of items $K'$ to be constant with $N$. We assume that
the fraction of users in each user class $\alpha_k$, the fraction of
items in each item class $\beta_{k'}$, the ratio between the number of
items and the number of users $\gamma = \frac F N$, as well as the
class affinity matrix $R = (r_{kk'})_{k,k'}$ are constant with respect
to $N$. Thus, the class sizes, as well as the total number of items
grow linearly with $N$.

We assume that the rating probability $p$ vanishes as the number of
users $N$ grows to infinity. Specifically, we again impose condition
\eqref{eq:omega} on parameter $\omega$. Note that the expected number
of items rated by a user is order of $\Theta(\omega)$.

We now formulate an extension of Theorem~\ref{thm:main} that we can
apply in this setting.

\subsubsection{Content Recommendation}
Consider vectors
$$
\alpha = (\alpha_{k})_{k=1}^{K}, \ \beta = (\beta_{k'})_{k'=1}^{K'},
$$
and define the following constant square matrix
$$
G := R \,\diag \beta R' \,\diag \alpha \in \mathds R^{K\times K}.
$$
We impose the following conditions, similar to \eqref{C14}:
\begin{subequations}
\label{D14}
\begin{align}
\label{D1} 
\parbox[b]{.85\textwidth}{The dimension of the profile space is upper bounded by $L \le \mbox{rank} (G)$.} \\
\label{D2}
\parbox[b]{.85\textwidth}{The $L$ largest eigenvalues of matrix $G$ are distinct.} \\
\label{D3}
\parbox[b]{.85\textwidth}{The corresponding eigenvectors $\{g_\ell\}_{\ell=1}^L$ normalised for the $\alpha$-norm satisfy:
$$
\chi_k \ne \chi_\ell, \ 1 \le k < \ell \le K,
$$
where $\chi_k = ( g_1(k), \dots, g_L(k) )$.} \\
\label{D4} 
\parbox[b]{.85\textwidth}{$\omega \ge C \log N$ for some constant $C$.}
\end{align}
\end{subequations}

We can now state the following
\begin{theorem}
\label{thm:rmain}
Under conditions~\eqref{D14}, with probability $1-o(1)$, a fraction of
$1-o(1)$ users $u$ is such that $\|\sqrt{N} z_u'-\chi_{k(u)}'\| = o(1)$.
That is to say, most users have their scaled profile vector close to a
fixed vector corresponding to their class.
\end{theorem}

We briefly explain how this seemingly distinct setup can be mapped to
the previous one. Define the following transformation which produces a
square matrix:
$$
\tau : S \mapsto A = \left[ 
\begin{array}{cc} 0 & S \\ 
S' & 0 
\end{array} 
\right]\in \mathbb R^{(N+F)\times(N+F)}, 
$$ 
where $S'$ denotes the transposition of matrix $S$. The spectrum of
such a matrix is symmetrical (i.e. if $\sigma$ is an eigenvalue of
$A$, then so is~$-\sigma$). Furthermore, the absolute values of the
eigenvalues of $A$ are the singular values of $S$ and its singular
components are determined by the eigenvectors of $A$.\footnote{To see this,
consider an eigenvalue $\sigma$ of $A$ and its corresponding
eigenvector $\zeta = \left[ \begin{array}{c} x \\ y \end{array}
  \right]$, with $x\in \mathbb R^{N\times 1}$ and $y\in \mathbb
R^{F\times 1}$. Since $A\zeta = \sigma \zeta$, we can write:
\begin{eqnarray*}
A \left[\begin{array}{c} x \\ y\end{array}\right] = \left[\begin{array}{c} Sy \\ S'x\end{array}\right] = \sigma \left[\begin{array}{c} x \\ y\end{array}\right] \mbox{ and } 
A \left[\begin{array}{c} x \\ -y\end{array}\right] = \left[\begin{array}{c} -Sy \\ S'x\end{array}\right] = -\sigma \left[\begin{array}{c} x \\ -y\end{array}\right],
\end{eqnarray*}
and thus $x'Sy = y'S'x = \sigma \|x\|^2 = \sigma \|y\|^2$. Then, we
have that $x$ is a left singular vector, and that $y$ is a right
singular vector for matrix $S$. For more details see for
instance~\cite{svd}.}

We do not give the proof of Theorem~\ref{thm:rmain}, since follows
closely that of Theorem~\ref{thm:main}. Essentially, we use the
transformation $\tau$ to obtain a square matrix, and subsequently we
apply Lemma~\ref{lemma:feige} and technical lemmas that we reproduce
in Appendix~\ref{app:rectlemmas} that play the role of
Lemmas~\ref{lemma:dominant} and~\ref{lemma:pert} in the proof.


\subsubsection{Characterising Performance of a Simple Voting Algorithm}

Let us now analyse a simple recommendation algorithm that relies on
local voting in the profile space. We have defined for each user $u$ a
scaled $L$-dimensional profile vector which we denoted by $\sqrt N
z_u'$. In Section~\ref{sec:drw} we propose a method for computing such
profile vectors in a distributed fashion based solely on local
information exchange.

Say we wish to characterise the taste of user $u$ for item
$f$. Consider a fixed constant $d > 0$. We define the {\em
  $d$-vicinity} of $u$ as the set of users that have profile vectors
at a Euclidean distance of at most $d$ from $\sqrt N z_u'$ in the
profile space. More formally,
$$
B(u,d) := \{v \in \mathcal U : \| z_u' - z_v'\|_2 \le \frac d {\sqrt N} \}.
$$

Votes for item $f$ are collected among users in $B(u,d)$ reflecting
their appreciation for item $f$. We denote the number of collected
votes by
\begin{equation}
\label{eq:estpr}
V_u(f) := \sum_{v \in B(u,d)} S_{vf}.
\end{equation}

The following proposition guarantees that, for well chosen $d$ and as
the number of users grows to infinity, from the perspective of most
users $u$, the quantities $V_u(f)$ provide an accurate ranking for
items with high probability.

Recall that by~\eqref{D4}, $\omega \ge C\log N$.

\begin{prop}
If $\rho_0 < \frac 1 6 \inf_{k,k'}\|\chi_k - \chi_{k'}\|$ and for a
large enough constant $C$, for all but a vanishing fraction of users
$u$, the following property is true: Consider a sample of items ${\cal
  I}\subset {\cal F}$ of constant size picked uniformly at
random. Then with high probability the item ranking given by the
$(V_u(f))_{f\in{\cal I}}$ coincides with the item ranking given by the
$(r_{k(u)k'(f)})_{f\in{\cal I}}$.
\end{prop}

\begin{proof}
  We consider a fixed user class $C_k$. For a specific item $f \in
  D_\ell$, define $V_k(f)$ to be the total number of votes it received
  from users in class $C_k$. It is distributed as a binomial:
  $\textrm{Bin}(\alpha_kN, r_{k\ell}\omega/N)$. We want $V_k(f)$ to be
  contained in the interval
  \begin{equation}
    \label{eq:vf}
    V_k(f) \in [\alpha_k r_{k\ell}\omega (1-\varepsilon_k/r_{k\ell}), \alpha_k
      r_{k\ell}\omega (1+\varepsilon_k/r_{k\ell})],
  \end{equation}
  for a given $\varepsilon_k$. A standard application of a Chernoff
  bound gives:
  $$ \PP {\left|\textrm{Bin}(\alpha_kN, r_{k\ell}\omega/N) -
    \alpha_k r_{k\ell}\omega\right| \ge \varepsilon_k \alpha_k \omega} \le
  2 e^{- \alpha_k r_{k\ell} \omega h(\varepsilon_k / r_{k\ell})},
  $$
  where $h(\cdot)$ is a convex increasing function such that
  $h(0)=h'(0)=0$.

  Let us generalise and define the following events:
  $$
  {\cal E}_{k,f} = \{| V_k(f)-\alpha_kr_{k\ell}\omega| < \varepsilon_k \alpha_k \omega\}, 
  $$
  where $\varepsilon_k$ are chosen such that
  $
  0 < \varepsilon_k = \inf_{\ell, \ell':|r_{k\ell} - r_{k\ell'}| \ne 0} \frac {|r_{k\ell} - r_{k\ell'}|} 3.
  $ 
  Then by definition, it must be that on ${\cal E}_k := \bigcap_f {\cal
    E}_{k,f}$ the following desirable ``separability'' property holds:
  for two items $f \in D_\ell$ and $f' \in D_{\ell'}$, if $r_{k\ell} >
  r_{k\ell'}$, then $V_k(f) > V_k(f')$. In other words, the votes of
  users of class $C_k$ preserve the inherent ranking of items.

  Let us bound the probability of the event ${\cal E}_k$:
  $$
  \mathds P({\cal E}_k) \ge 1 - \sum_\ell |D_\ell| 2 e^{- \alpha_k r_{k\ell} \omega h(\varepsilon_k / r_{k\ell})}.
  $$

  We impose the following condition on the constant $C$ of~\eqref{D4}:
  $$
  C \ge F_1( \{\alpha_k\}, \{r_{k\ell}\} ),
  $$
  where
  $$
  F_1( \{\alpha_k\}, \{r_{k\ell}\} ) = \sup_{k,\ell} \frac 1 {\alpha_k r_{k\ell}h(\varepsilon_k/r_{k\ell})}.
  $$
  Then, since $|D_\ell| = \Theta(N)$, it follows that all the events
  ${\cal E}_k$ occur with high probability (i.e., with probability
  converging to $1$ as the number of users $N$ goes to infinity).
  We have proved the following
  \begin{lemma}
    If $\omega \ge C \log N$ and $C \ge F_1( \{\alpha_k\}, \{r_{k\ell}\} )$,
    then with high probability for any class $C_k$ and any two items $f
    \in D_\ell$ and $f' \in D_{\ell'}$ such that $r_{k\ell} > r_{k\ell'}$,
    we have that $V_k(f) > V_k(f') + C''\omega$ for some positive constant
    $C''$.
  \end{lemma}

  Denote the total number of votes of user $u \in C_k$ by $V_u$. It is
  distributed as a sum of binomials: $\sum_\ell \textrm{Bin}(|D_\ell|,
  \frac \omega N r_{k,\ell} )$. The expected number of votes is thus
  $\mathds E (V_u) = \sum_\ell |D_\ell| \frac \omega N r_{k\ell} =
  \Theta(\omega) = m_k\omega$, where $m_k = \sum_\ell \beta_\ell
  r_{k\ell}$.

  By a similar Chernoff bound argument for a given $\delta > 0$ we
  find
  $$
  \mathds P( V_u \ge m_k \omega + \delta \omega ) \le
  e^{-m_k\omega h(\frac \delta {m_k + \delta})}.
  $$
  Hence,
  $$
  \mathds P( \exists u : V_u \ge m_{k(u)} \omega + \delta \omega ) \le
  \sum_k |C_k| e^{-m_k\omega h(\frac \delta {m_k + \delta})}.
  $$

  Define
  $$
  F_2( \{m_k\} ) = \sup_{k} \frac 1 {m_k h(\frac \delta {m_k + \delta})}.
  $$

  The above probability goes to $0$ as $N$ goes to infinity under the
  additional conditions on the constant $C$ stated in the following
  \begin{lemma}
    Denote $\bar m = \max_k m_k$. If $C \ge F_2( \{m_k\} )$, all users
    $u$ verify with high probability
    $$
    V_u \le (m_{k(u)} + \delta)\omega \le (\bar m + \delta)\omega.
    $$
  \end{lemma}

  Let us now characterise the performance of the distributed voting
  scheme. By Assumption~\eqref{D3}, the constant vectors corresponding
  to the user classes $(\chi_k)_k$ are such that $\|\chi_k -
  \chi_{k'}\| = \Omega(1)$, for any $1\le k \ne k' \le K$.

  By hypothesis we have that $\rho_0 < \frac 1 6 \inf_{k,k'}\|\chi_k -
  \chi_{k'}\|$. A user~$u$ of class~$C_k$ queries other users having
  profile vectors within the ball $B(\sqrt N z_u, 2\rho_0)$ about some
  item~$f$. Thus, for any $u$ such that $\|\sqrt N z_u-\chi_{k}\| \le
  \rho_0$, all users in $B(\chi_{k}, \rho_0)$ are necessarily
  queried. {\em Henceforth we only take interest in such users}, as
  they constitute a fraction of the total number of users that goes to
  $1$ as $N$ goes to infinity.

  We denote by $V_u(f)$ the number of votes collected by $u$ for
  $f$. Then the difference $W_u(f) := V_u(f) - V_k(f)$ is the error that
  $u$ makes when estimating $V_k(f)$. It can be seen as the difference
  between the unwanted votes of users in various classes $C_{k'}$ whose
  profile vectors fall within $B(\sqrt N z_u, 2\rho_0)$ (and hence fall
  outside the ball of radius $\rho_0$ around the constant vector
  corresponding to their own class $B(\chi_{k'}, \rho_0)$) and the votes of
  users in $C_k$ who fall outside $B(\sqrt N z_u, 2\rho_0)$ (and implicitly
  outside of $B(\chi_{k}, \rho_0)$).

  We can bound $\sum_f |W_u(f)|$ as follows:
  \begin{align*}
    \sum_f |W_u(f)| &\le \mathop{\sum_{u'\in C_k:}}_{\sqrt N z_{u'} \not \in B(\chi_k,\rho_0)} V_{u'} + \mathop{\sum_{u'\not \in C_k:}}_{\sqrt N z_{u'} \not \in B(\chi_{k(u')},\rho_0)} V_{u'} \\
    &\le N_{out} \cdot (\bar m + \delta) \omega,
  \end{align*}
  where $N_{out} = |\{ u' : \sqrt N z_{u'} \not \in B(\chi_{k(u')},\rho_0)
  \}| = o(N)$, as shown in Theorem~\ref{thm:rmain}.


  We call an item $f$ ``misclassified'' by $u$ if 
  $$ 
  V_u(f) = V_k(f) + W_u(f) \not \in [\alpha_k \omega
    (r_{k\ell}-\varepsilon_k-\varepsilon_k/2), \alpha_k \omega
    (r_{k\ell} + \varepsilon_k+\varepsilon_k/2)].
  $$

  If this is the case in ${\cal E}_k = \bigcap_f {\cal E}_{k,f}$, then
  necessarily $|W_u| \ge \alpha_k \omega \frac {\varepsilon_k} 2$, and
  the number of misclassified items is upper bounded as
  $$
  |\{f \mbox{ misclassified by $u$}\}| \le \frac {\sum_f |W_u(f)|}
  {\alpha_k \omega \frac {\varepsilon_k} 2} \le
  \frac {\bar m + \delta} {\alpha_k \frac {\varepsilon_k} 2} N_{out} = o(N).
  $$

  Thus, a constant sample of items chosen at random are well ranked
  (i.e., the order of the $V_u(f)$ is consistent with the order of the
  $r_{k\ell}$) with high probability for all users $u$ (within the
  same class of items the ordering is irrelevant).
\end{proof}

\subsection{Spectral Profiling in Practice}
\label{sec:speval}

In this section we evaluate the benefits of spectral techniques on a real trace provided by Netflix. The Netflix data set contains about $10^8$ user ratings for $17,770$ movies by $480,000$ users. The ratings are given in the form of an integer number of stars, ranging from $1$ to $5$.


We evaluate taste similarity between users that are assigned close-by profiles in the spectral embedding. We do this as follows: We select a set of $2000$ users and a set of $2000$ movies from the Netflix data set, such that the selected users have given roughly the same number of ratings within the selected movie set. The presence of a rating is viewed in this setting as a sign that the user has viewed that particular movie, and is therefore considered as a binary form of appreciation (the lack of a rating denoting a potential lack of interest for that content). Subsequently, for each user we hide the rating of one content at random. Using the remaining observed ratings, we build a sparse observed $0-1$ rating matrix $S$, and we compute the spectral profiles $\sqrt N z_u'$ of the users. For each user, an ordered list of neighbours is implicitly defined, from the closest to the farthest one in the profile space. We compute over the set of users the average frequency of the occurence of the following event: ``a user at distance $k$ has rated the hidden content''. The average is taken over the set of users. This ``frequency of agreement'' reflects the taste proximity of users.

\begin{figure}
\centering
\subfigure[Average frequency of agreement with nearest neighbour]{
\includegraphics[width=0.46\textwidth]{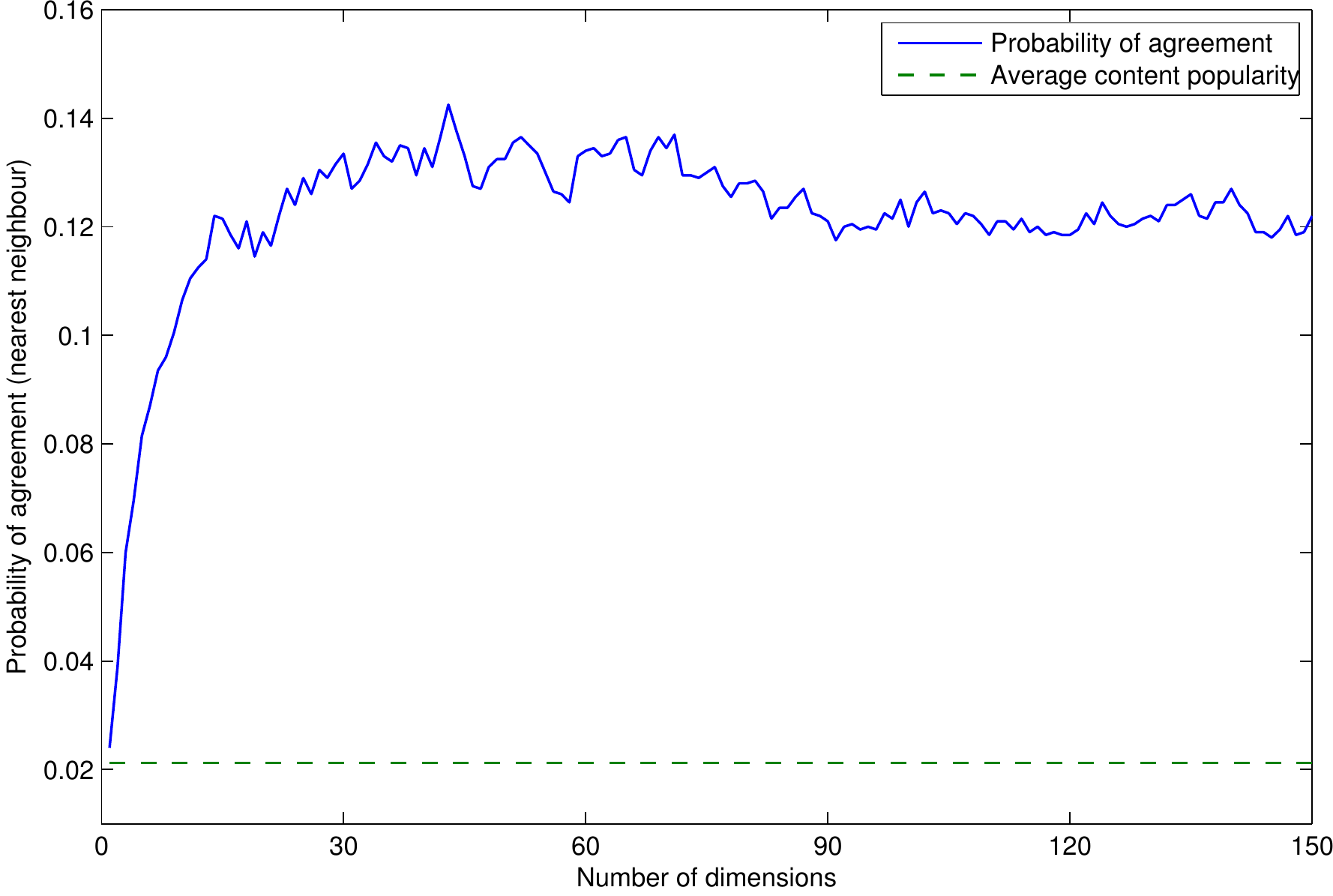}
\label{fig:dim-d1}
}
\subfigure[Average frequency of agreement with neighbours]{
\includegraphics[width=0.46\textwidth]{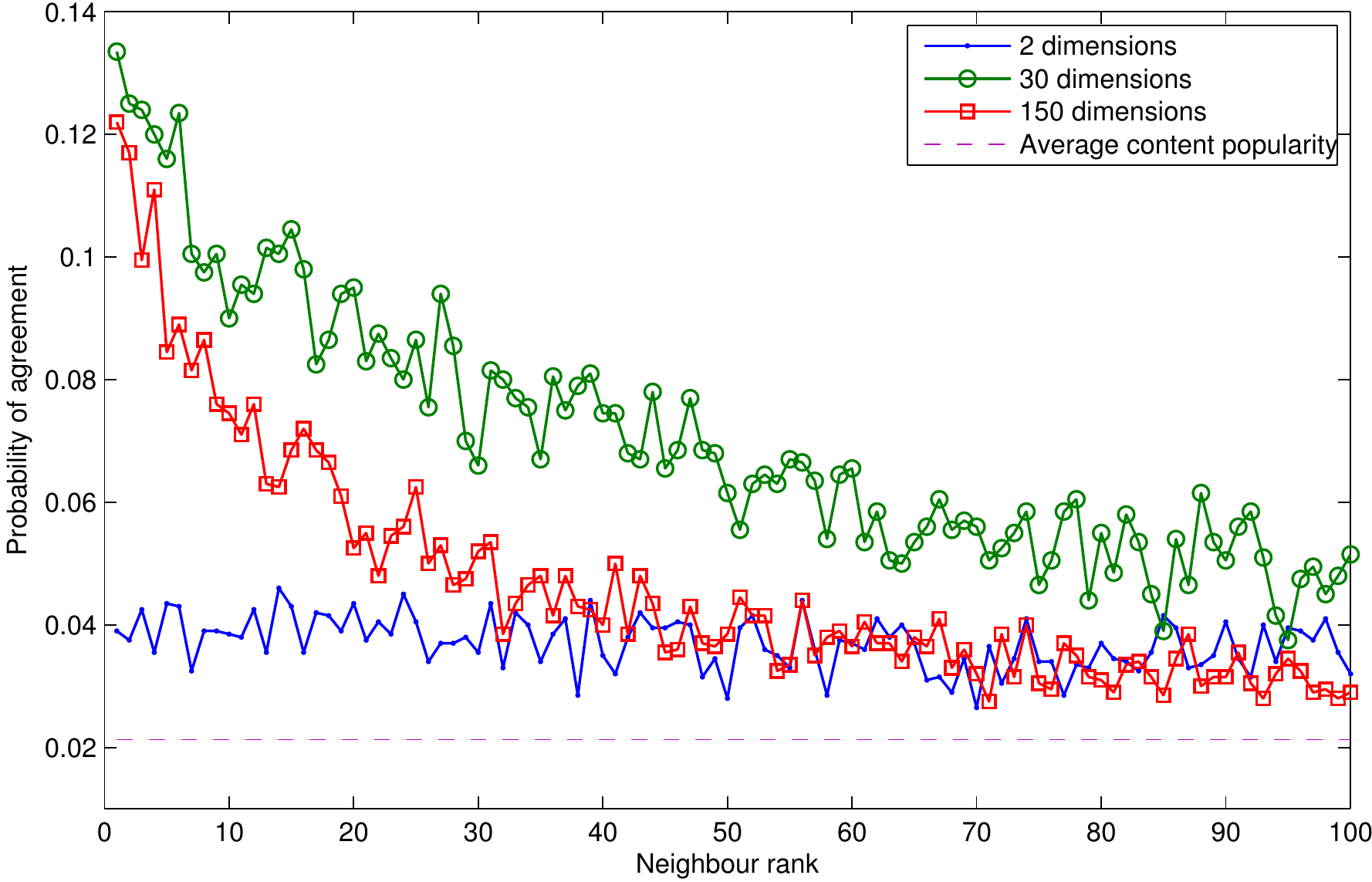}
\label{fig:rank}
}
\caption{Profiling Proximity}
\label{fig:embed}
\end{figure}

In Figure~\ref{fig:embed} we plot this ``frequency of agreement'' for different values of the dimensionality $L$ of the embedding. Content popularity (i.e. the fraction of users having rated it) ranges from $0.45\%$ to $5.3\%$, and the average popularity of content is $2.13\%$. We show this value on our plots for comparison. In Figure~\ref{fig:dim-d1} we vary the dimension from $2$ to $150$ and plot the average frequency of agreement with the nearest neighbour. We notice that there is a peak in this average frequency around roughly $30$ dimensions. Subsequently the plotted frequency decays slowly. In Figure~\ref{fig:rank} we plot the average frequency of agreement for $2$, $30$ and $150$ dimensions for the $100$ closest neighbours in the profile space. We conclude that an embedding of rank $30$ is appropriate to characterise user taste for the selected users in the trace.

It is important to observe that the frequency decays with distance in the profile space. This indicates that spectral profiling manages to capture user taste proximity.

\section{Oja's Algorithm and Beyond}
\label{sec:drw}
We now propose a method for extracting the eigenvectors of the
adjacency matrix of a graph in a distributed manner. Eigenvector
extraction is the object of Oja's algorithm~\cite{oja}, the basic
version of which is refined by Borkar and Meyn in~\cite{bmoja}.

Consider a sequence of symmetric square random matrices $(A_k\in
\mathbb R^{N\times N})_k$ of common finite mean $A\in \mathbb
R^{N\times N}$. Oja and Karhunen~\cite{oja} proposed the following
stochastic approximation algorithm for determining the $s$ top
eigenvectors of $A$:
\begin{align}
\tilde X_k &= X_{k-1} + A_k X_{k-1} \Gamma_k, \label{eq:oja-1} \\
X_k &= \tilde X_k R_k^{-1}, \label{eq:oja-2}
\end{align}
where $X_k$ is an $N\times L$ matrix, $\Gamma_k$ is a diagonal matrix
of gains and $R_k^{-1}$ is a matrix achieving the orthonormalisation
of the columns of $\tilde X_k$. They prove almost sure convergence
under typical assumptions on the sequence of gains, assuming unit
multiplicity of the top $s$ eigenvalues, probability density uniformly
bounded away from $0$ for each of the $A_k$, and almost sure
boundedness and symmetry of the $A_k$, as well as statistical
independence.

A simpler single-step form of the algorithm was also proposed by Oja
and Karhunen; the latter produces only an orthonormal basis of the
subspace spanned by the top $s$ eigenvectors. The convergence of a
slightly modified version thereof when $A$ is positive definite
was showed by Borkar and Meyn~\cite{bmoja}, who introduced the
additional factor $\frac 1 {1+\mbox{Tr}(X_kX_k')}$ and added the
additional i.i.d. $\mathcal N(0,I)$ noise sequence $\xi_k$:
$$
X_k-X_{k-1} = \frac {a_k} {1+\mbox{Tr}(X_kX_k')} [(I - X_{k-1}X_{k-1}')A_{k-1}X_{k-1} + \xi_k].
$$
Here $(a_k)$ is an almost typical gain sequence: 
$$
\sum_k a_k = \infty, \quad \sum_k a_k^2 < \infty, \quad \sup_k \frac{\sum_{n\ge k}a_n^2}{a_k} < \infty.
$$

None of these approaches are distributed however. Our contribution
consists precisely in augmenting these methods to make them
distributed.

\subsection{A Method for Distributed Spectral Profiling}
\label{sec:distro}
In Section~\ref{sec:theory} we have demonstrated the benefits of
spectral profiling. We have seen that, given a set of $N$ users, it is
sufficient that each of them contact on average only $\Omega(\log(N))$
other users at random and determine similarity with them to
essentially manage to characterise the profiles of everyone, by
application of the spectral transformation.

In this section we consider a (sparse) graph obtained via such a
probing process. We develop message-passing algorithms that enable all
users to individually compute their spectral profile, while allowing
communication only between pairs of neighbours. The algorithms we
propose in this section only require the connectivity property of the
graph (and hence no specific bound on neighbourhood sizes). However,
the scenario of a sparse graph is most appealing, as it fits very well
with the model we introduced in the previous section.

Let us thus consider a network described by an undirected graph
$\mathcal G = (\mathcal U, \mathcal E)$, where $\mathcal U$ is the set
of $N$ nodes (also designated interchangeably by ``users''), and
$\mathcal E$ is the set of edges connecting these nodes. We denote by
${\mathcal N}_u$ the set of neighbours of $u$, and also write $u\sim
v$ to indicate that two nodes $u$, $v$ are neighbours. Nodes that
share a link compute their ``similarity'' value, i.e., a real number
$A_{uv}$ which reflects pairwise taste similarity (in
Section~\ref{sec:theory} we were only considering binary values). The
computed values define a weighted (undirected) graph with adjacency
matrix $A$.

In this section we show that nodes are able to compute individually
via message passing their coordinates in an $L$-dimensional profile
space. These coordinates form a collection of $L$ linearly independent
vectors which span the vector space generated by the $L$ eigenvectors
corresponding to the top $L$ largest magnitude eigenvalues of $A$.
For binary similarity values, if we interpret the aforementioned
computed coordinates as user profiles, Theorem~\ref{thm:main} and its
corollaries still apply under the same assumptions~\eqref{C14}; thus,
for the considered class-based model of user taste, clusters
corresponding to the different classes still emerge. We show this in
Appendix~\ref{app:port}.

Let us now describe the proposed method for distributed user profile
computation.

By definition $A_{uu} = 0$ (and thus $\mbox{Tr}(A) = 0$), thus the
matrix is not positive semidefinite (it has necessarily negative
eigenvalues). We need to alter the matrix $A$ to guarantee positive
semidefiniteness, without changing its eigenvectors. To achieve this,
pick for example either one of the two solutions below:
\begin{itemize}
  \item[-] Compute the value $\Delta := \max_u \sum_{v}|A_{uv}|$
    (e.g., via a distributed voting scheme), and subsequently set the
    diagonal values of $A$ to a value larger than $\Delta$, e.g.,
    $A_{uu} := \Delta + \epsilon$ for some $\epsilon>0$. This
    procedure simply adds $\Delta + \epsilon$ to the eigenvalues of
    $A$, thus rendering them positive ($\Delta$ corresponds to the
    maximum degree of the graph in the binary case, and it is known
    that $|\lambda_1| < |\Delta|$).
  \item[-] Use matrix $A^2$, which has the same eigenvectors as $A$
    and eigenvalues equal to the square of those of $A$. The advantage
    of this second method is that the ranking of the magnitudes of the
    eigenvalues is preserved. However, to avoid direct communication
    with distance-$2$ neighbours, local gossiping should be used (as
    described in Section~\ref{sec:disteval}).
\end{itemize}

In the rest of this section we consider that one of the solutions
above is used and thus that matrix $A$ is rendered positive
semidefinite.

Given some fixed number $L$ of target eigenvectors to be extracted,
each user $u$ maintains at all time $t$ three sets of variables:
\begin{enumerate}
\item An $L$-dimensional row vector $X_u(t)$, the sought-for eigenvector coordinates;
\item An $L\times L$ matrix $\Phi_u(t)$ and a  scalar $\Psi_u(t)$, both playing an auxiliary role in the calculation.
\end{enumerate}
For ease of presentation, we assume slotted time $t=0,1,2,\ldots$, and synchronous updates at all peers. Asynchronous versions will be described and tested in the next section.

Our algorithm then takes the following form:
\begin{equation}
 X_u(t+1) - X_u(t)=\frac{ a(t)}{Y_u(t)}\left[\sum_{v\sim u}A_{vu}X_v(t) - N X_u(t) \Phi_u(t) + \xi_u(t+1)\right].
\label{dist1}
\end{equation}
In the above, $a(t)$ is a gain parameter to be specified,  $\xi_u(t+1)$ is a noise term deliberately introduced by user $u$, and the denominator $Y_u(t)$ is taken equal to 
\begin{equation}\label{dist2}
Y_u(t)=\max\left(1, |\Psi_u(t)|,\frac{1}{N L^2}\sum_{k,\ell=1}^L |(\Phi_u(t))_{k\ell}|\right).
\end{equation}
It is readily seen that the update (\ref{dist1}) can be computed locally at $u$, solely relying on variables local to node $u$ and inputs $X_v(t)$ from $u$'s neighbours $v\in {\mathcal  N}_u$. The same is also true for the updates of variables $\Phi_u$ and $\Psi_u$, which take the forms:
\begin{equation}\label{dist3}
\Phi_u(t+1)=\Phi_u(t)+b(t)\sum_{v\sim u}\left(\Phi_v(t)-\Phi_u(t)\right)+f_u(t+1)-f_u(t),
\end{equation}
where $b(t)$ is a gain parameter, $f_u(t)$ is a $L\times L$ matrix, specified by
\begin{equation}\label{dist4}
f_u(t)=X'_u(t) \sum_{v\sim u} A_{uv} X_v(t),
\end{equation} 
and 
\begin{equation}\label{dist5}
\Psi_u(t+1)=\Psi_u(t)+b(t)\sum_{v\sim u}\left(\Psi_v(t)-\Psi_u(t)\right)+g_u(t+1)-g_u(t),
\end{equation}
where $g_u(t)$ is a scalar, specified by
\begin{equation}\label{dist6}
g_u(t)=X_u(t) X'_u(t).
\end{equation} 
Before stating the main result of this section, we introduce the technical conditions that will be required from the gain sequences $a(t)$, $b(t)$:
\begin{subequations}
\label{cond1}
\begin{align}
\label{i} & a(t),\; b(t)\in[0,1],\; t>0,\\
\label{ii} & \sum_{t> 0}a(t)=\sum_{t> 0} b(t)=+\infty,\\
\label{iii} & \sum_{t> 0}a(t)^2<+\infty,\; \sum_{t> 0}b(t)^2<+\infty,\\
\label{iv} & \lim_{t\to\infty}\frac{a(t)}{b(t)}e^{K\sum_{s=1}^{t}a(s)}=0,\; K>0
\end{align}
\end{subequations}
Note that these conditions are satisfied for instance upon taking $a(t)=1/(t\log(t))$, and $b(t)=t^{-2/3}$.
Indeed, with this choice for $a(t)$, it is readily seen that 
$$
\sum_{s=1}^t a(s)\sim \log(\log(t)),
$$
and~\eqref{ii} follows. In addition, the quantity in~\eqref{iv} then reads
$$
\frac{1}{t^{1/3} \log(t)}e^{K(1+o(1))\log(\log(t))}\le \frac{(\log(t))^{2K-1}}{ t^{1/3}},
$$
where we have used the upper bound of 1 on the term $o(1)$, and property~\eqref{iv} readily follows.

We are now  in a position to state this section's main result:

\begin{theorem}\label{thm11}
Assume that the gains $a(t)$, $b(t)$ verify the
conditions~\eqref{cond1}. Assume further that the noise terms
$\xi_u(t)$ are i.i.d, white Gaussian noise. Assume finally that the
overlay graph over which peers communicate is connected, and that
matrix $A$ is positive semidefinite (see discussion above). Then the
distributed updating algorithm (\ref{dist1}--\ref{dist6}) verifies the
following property: With probability 1, the columns of
$X(t):=(X_u(t))_{u\in{\mathcal U}}$ converge to a collection of $L$
orthonormal vectors spanning the vector space associated with the $L$
largest eigenvalues of $A$.
\end{theorem}
The proof of the theorem is given in Section~\ref{proof:thm11}. In what follows, we provide some background and intuition for it.

\index{Algorithm!Oja's}
Consider first the main equation, (\ref{dist1}). If we ignore the denominator $Y_u(t)$, the noise term $\xi_u(t+1)$, and replace the term $N X_u(t) \Phi_u(t)$ by $\sum_v X_u(t)f_v(t)$, where $f(t)$ is as given in (\ref{dist4}), this equation reads, written in matrix form:
$$
X(t+1)-X(t)=a(t)\left[ A X(t)- X(t)X'(t)A X(t)\right].
$$
This is in fact the celebrated Oja algorithm, proposed by Oja and Karhunen~\cite{oja} to extract precisely the eigenvectors of the largest eigenvalues of $A$. Oja's algorithm is subject to some stability issues, that Borkar and Meyn~\cite{bmoja} proposed to fix by scaling down the right-hand side of the previous equation by some factor $Z(t)=1+\sum_{u,k}X_{u,k}^2(t)$, and by adding an extra noise term $\xi(t+1)$ in the bracket in the right-hand side. Thus, the update rule they considered reads:
\begin{equation}
\label{borkar-meyn}
X(t+1)-X(t)=\frac{a(t)}{Z(t)}\left[ A X(t)- X(t)X'(t)A X(t)+\xi(t+1)\right], 
\end{equation}
and is proved in (\cite{bmoja}) to converge with probability 1 to the desired eigenvectors, under assumptions~\eqref{ii},\eqref{iii} on the gains $a(t)$, and similar conditions on the noise $\xi(t)$ as in our theorem.

However, algorithm~(\ref{borkar-meyn}) does not lend itself to a distributed implementation, since neither of the two terms $X'(t)AX(t)$ or $Z(t)$ can be computed locally by the users. 

To solve this issue, we introduce the auxiliary local variables $\Phi_u$, $\Psi_u$. The dynamics (\ref{dist3}--\ref{dist5}) according to which they evolve is best understood by setting to zero the input terms $f_u(t+1)-f_u(t)$ and $g_u(t+1)-g_u(t)$ in the right-hand side. It then becomes apparent that these dynamics perform local averaging (also known as gossiping in \cite{boyd}). Thus these eventually converge to a state where all variables $\Phi_u(t)$ coincide with the average $(1/N)\sum_v \Phi_v(0)$ of the original entries.

We can now provide a heuristic argument for the theorem. On a fast time scale, characterised by the gain parameters $b(t)$, the gossiping dynamics converge to almost constant vectors, with
$$
\begin{array}{ll}
\Phi_u(t)\equiv \frac{1}{N}\sum_{v}f_v(t), &u\in {\mathcal U},\\
\Psi_u(t)\equiv \frac{1}{N} \sum_{v}g_v(t) &u\in {\mathcal U}.
\end{array}
$$
Then on a slower time scale dictated by the gain parameters $a(t)$, the variables of interest $X_u(t)$ follow dynamics very close to (\ref{borkar-meyn}). Indeed, the auxiliary parameters $\Phi_u$, $\Psi_u$ track accurately the desired terms $X'(t) A X(t) $ and $Z(t)$ respectively. 

A  couple of remarks are in order. The stabilisation by the scaling factor $Z(t)$ in (\ref{borkar-meyn}) seems insufficient in the presence of the additional dynamics (\ref{dist3},\ref{dist5}). This leads us to introduce our alternative stabilisation via $Y_u(t)$ in (\ref{dist2}). Also, in problems with dynamics at two time scales a common assumption on the gain parameters is that $a(t)/b(t)\to 0$. In the present case, a stronger form of time scale separation (namely, condition~\eqref{iv}) is needed, to prevent reinforcing instabilities between the two dynamics.

\subsection{Evaluation of an Asynchronous Version}
\label{sec:disteval}
In this section we present numerical evaluations on synthetic data. We exhibit convergence of an asynchronous version of the distributed coordinate assignment scheme presented in the previous section on synthetic data generated according to the model presented in Section~\ref{sec:sq}.

In Section~\ref{sec:distro} we showed that the distributed algorithm~(\ref{dist1}--\ref{dist3}) converges almost surely towards $L$ linearly independent vectors spanning the vector space generated by the eigenvectors corresponding to the top $L$ magnitude eigenvalues of the adjacency matrix $A$.

In the following, we evaluate the asynchronous version of the algorithm. In this setting, each node $u$ keeps track of its own coordinates $X_u$ as well as the gossiped variables $\Phi_u$ and $\Psi_u$. However, instead of explicitly imposing a timescale separation via gains $a(t)$ and $b(t)$ while enforcing a synchronised evolution of all the quantities, we impose distinct rates at which the updates are performed. Namely, the coordinate updates~(\ref{dist1}) are performed independently according to Poisson processes of rate $\lambda$, while gossiping~(\ref{dist2}-\ref{dist3}) is performed independently ``pairwise'' according to Poisson processes of rate $\mu \gg \lambda$. By pairwise we mean that a pair of nodes $(u,v) \in \mathcal E$ will exchange and update their values for $\Phi$ and $\Psi$ at rate $\mu$ similarly to the randomised gossiping technique from~\cite{boyd}.

Furthermore, we replace the adjacency matrix by its square $A^2$. We choose to do so, since the latter is positive semidefinite, has the same eigenvectors as $A$, and a spectrum composed of the squared eigenvalues of $A$. Implicitly, the eigenvectors are ordered according to the magnitude of the eigenvalues of $A$, instead of their actual values. In turn, this modification alters the function $f_u$ from~(\ref{dist4}), which becomes:
\begin{equation}
\label{eq:fmod}
f^{(2)}_u = \left(\sum_{v\sim u}A_{uv} X_v\right)' \left(\sum_{v\sim u}A_{uv} X_v \right).
\end{equation}
The algorithm executed at each node is summarised in Algorithm~\ref{alg:distro_mod}.

\index{Algorithm!Distributed spectral profiling}
\begin{algorithm}[ht]
  \caption{Distributed profiling algorithm}
  \label{alg:distro_mod}

  \begin{minipage}[b]{.47\textwidth} 
  {\small \textsc{Node($u$)::Update-Local}()} at rate $\lambda$ \\
  {\small Local Variables: $X_u, \Pi_u, w_u, X^0_u, \Pi^0_u, \Psi_u, \Phi_u$}\\[-.5cm]
  \begin{algorithmic}[1]
    \IF{$w_u = 1$}
    \STATE Retrieve partial product vectors $\Pi_v$ from $v\in \mathcal N_u$
    \STATE $X_u := X_u + \gamma \frac{\sum_{v\sim u} A_{uv}\Pi_v - N X_u \Phi_u}{Y_u(\Psi_u, \Phi_u)}$ 
\label{algst:upd}
    \ELSE
    \STATE Retrieve vectors $X_v$ from $v\in \mathcal N_u$
    \STATE $\Pi_u := \sum_{v\sim u} A_{uv}X_v$ \label{algst:pi}
    \ENDIF
    \STATE $w_u := 1 - w_u$
  \end{algorithmic}
  \end{minipage}
  \hfill
  \begin{minipage}[b]{.47\textwidth} 
  {\small \textsc{Link($u,v$)::Gossip}()} at rate $\mu$ \\
  \vspace{-13pt}
  \begin{algorithmic}[1]
  	\STATE Retrieve local variables at $u$ and $v$
  	\FOR{$(H,h)$ in $\{(\Phi, f^{(2)}),(\Psi,g)\}$}
  		\STATE $\alpha := \frac {H_u + H_v} 2$
  		\FOR{$i$ in $\{u, v\}$}
  			\STATE $\delta_i := h_i( X_{i}, \Pi_i ) - h_i( X^0_{i}, \Pi^0_i )$
  			\STATE $H_i := \alpha + \delta_i$
  		\ENDFOR
  	\ENDFOR
  	\STATE $X^0_{u} := X_{u}$, $\Pi^0_{u} := \Pi_{u}$, $X^0_{v} := X_{v}$, $\Pi^0_{u} := \Pi_{u}$\\
  \end{algorithmic}
  \end{minipage}
\end{algorithm}

Since sparsity is not preserved by taking the square of $A$, we cannot simply use $A^2$ as a new adjacency matrix. Thus, we need to compute products $A^2X$ specifically: For some node $u$, every other call to the {\small \textsc{Update-Local}()} procedure computes the partial products $\Pi_u = \sum_{v\sim u} A_{uv}X_v$ (State~\ref{algst:pi} in the Algorithm). Subsequently, the neighbours' partial product vectors $\Pi_v$ are used in the coordinate update procedure~(\ref{dist1}) at State~\ref{algst:upd} and for the gossiping of $f^{(2)}_u = \Pi'_u \Pi_u$. Each node additionally stores its previous vectors $X^0_u$ and $\Pi^0_u$ for use in the {\small \textsc{Gossip}()} procedure.

The only piece of global information required is the number of users in the system (or an approximation thereof). In the {\small \textsc{Update-Local}()} procedure, we used a fixed gain $\gamma$. The noise component $\xi$ is omitted in the algorithm. Even so, noise is intrinsic to the algorithm as it is introduced by both the gossip averaging and by the fact that exchanges are asynchronous.
\begin{figure}
\centering
\includegraphics[width=.6\textwidth]{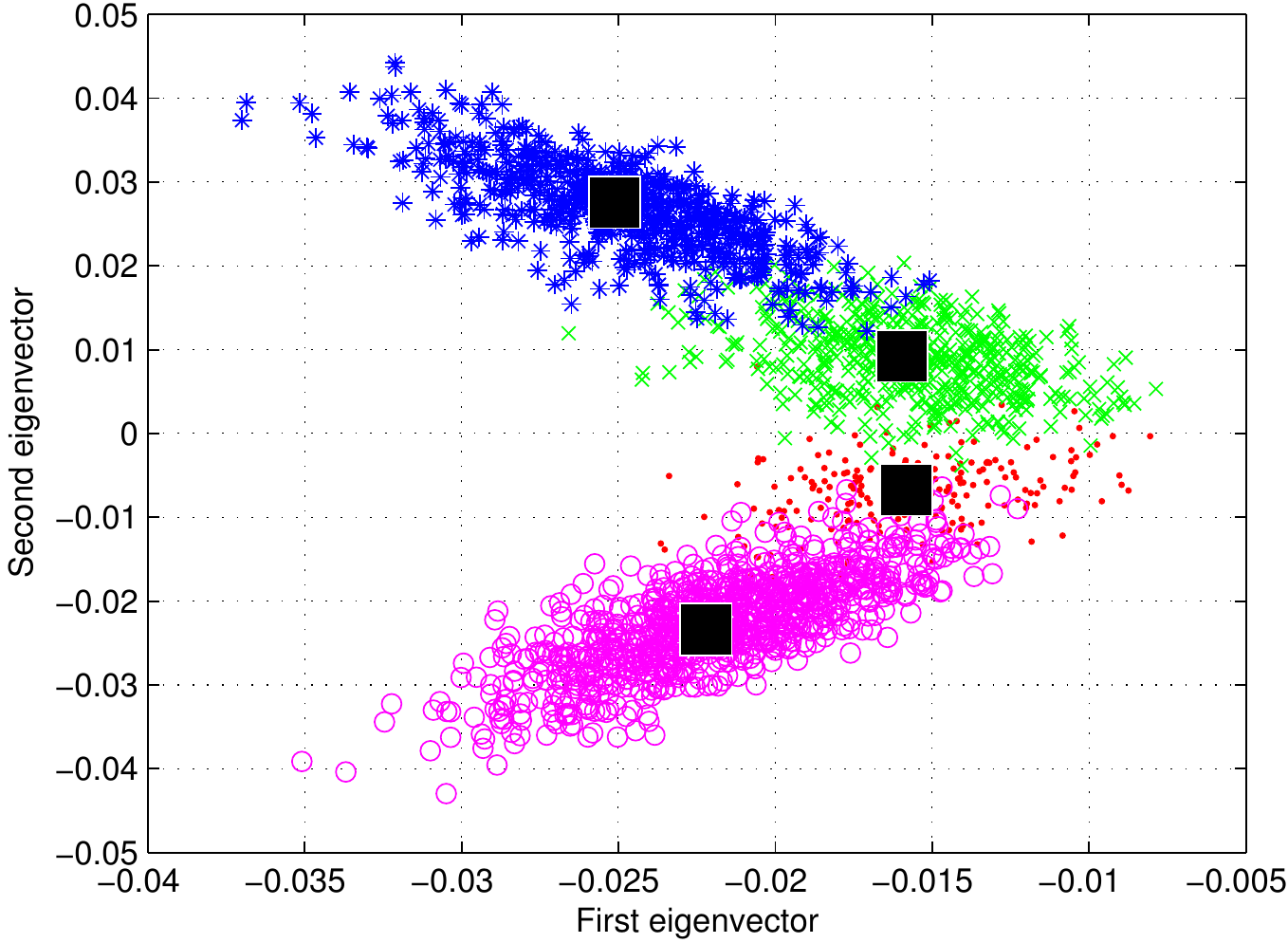}
\caption{Pure and noisy user profiles}
\label{fig:clx}
\end{figure}

For evaluation of convergence, we used a synthetic data trace generated according to our model from Section~\ref{sec:sq}. This trace considers $2200$ users clustered in $4$ classes of sizes $200, 500, 600$ and $900$. The probability matrix $pB$ characterising the classes is
$$
pB = \frac 1 {100} \left(\begin{array}{cccc}
  0.5 & 1.5 & 2.0 & 1.0 \\
  1.5 & 0.55 & 1.0 & 2.0 \\
  2.0 & 1.0 & 0.45 & 4.0 \\
  1.0 & 2.0 & 4.0 & 0.5
\end{array}\right).
$$
We consider an adjacency matrix $A$ generated using the aforementioned parameters. This matrix is sparse (the average node degree is $40$). For visualisation ease, we consider the case $L = 2$ (i.e. the profile space is a 2-dimensional plane). In Figure~\ref{fig:clx} we plot the first two eigenvectors of the expected adjacency matrix $\bar A$. They are constant on each of the $4$ classes, hence the plot is constituted of $4$ points, which are depicted as the four black squares. They represent the pure user profiles which characterise each of the four classes. Additionally, we plot the first two eigenvectors of the ``noisy'' adjacency matrix $A$, with elements belonging to a specific class marked with the same symbol. We have the visual confirmation that despite the sparsity of matrix $A$, the noisy profiles are grouped around the pure profiles. 

We choose a gain $\gamma = 0.001$, an update rate $\lambda = 0.2$ and a gossip rate $\mu = 10$. We initialise the algorithm at a random state. In Figure~\ref{fig:traj} we plot the time evolution of the proportion of the mass of the two coordinate vectors $X_{\cdot1}$ and $X_{\cdot2}$ (aggregated across users) that falls on the space orthogonal to the 2-dimensional eigenspace generated by the first two eigenvectors of matrix~$A^2$. Additionally, we plot the scalar product of the two coordinate vectors. After roughly $400$ time units, we observe convergence towards orthogonal vectors spanning the desired eigenspace.

\begin{figure}
\centering
\includegraphics[width=.6\textwidth]{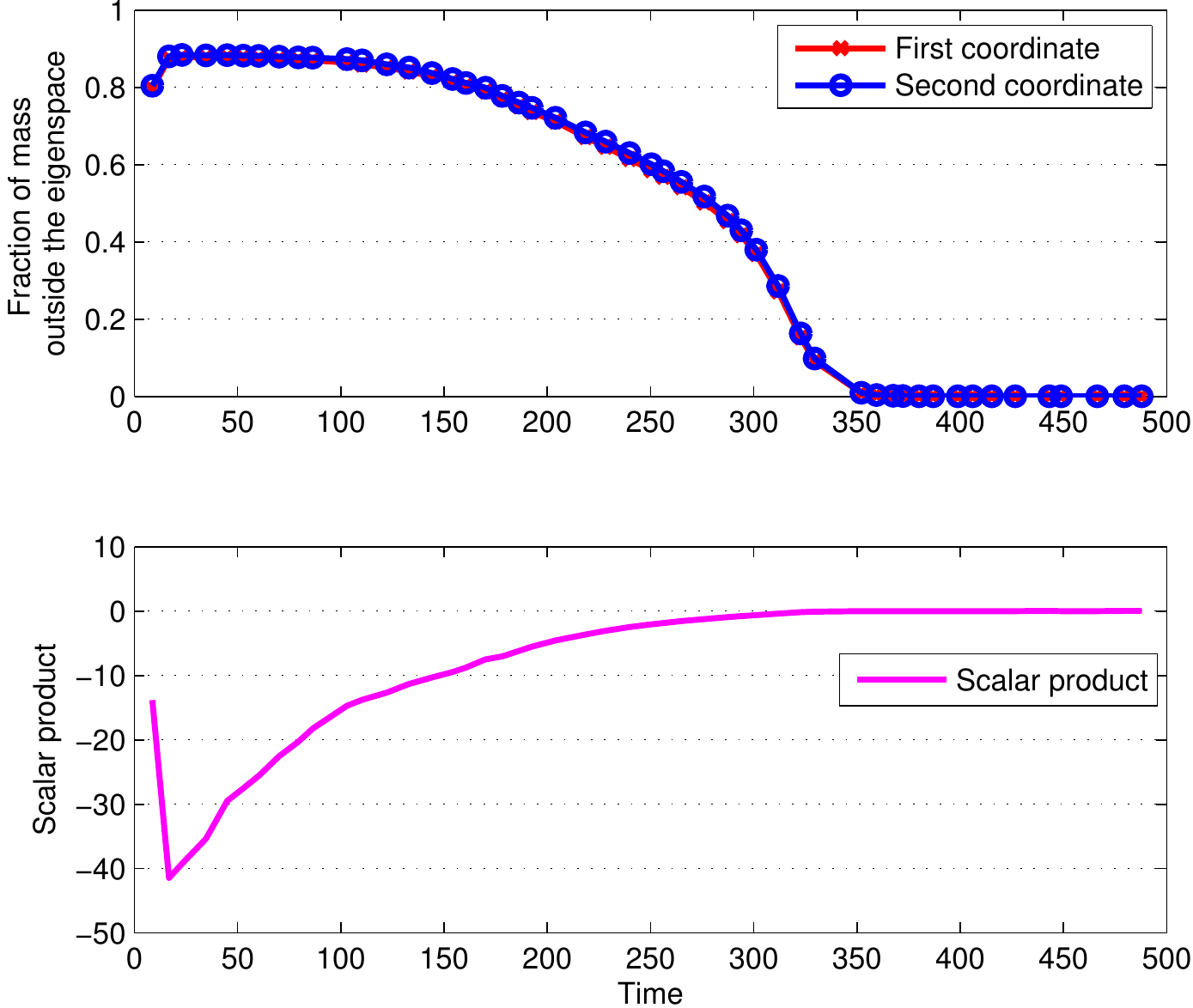}
\label{fig:traj}
\caption{Convergence of the asynchronous distributed algorithm}
\end{figure}

\section{Conclusions}
\label{sec:concl}
In this paper we addressed the problem of distributed user profiling and recommendation.

We first showed that spectral techniques constitute an appealing approach, and obtained novel results on their efficiency, thereby improving upon previous literature on the subject of spectral clustering. We showed that a for a low-rank probabilistic model of user taste, a simple distributed algorithm based on local votes in the profile space asymptotically achieves accurate prediction of user preference.

We developed techniques for computing eigenvectors in a distributed manner. Our solution combines ideas from Oja's algorithm with gossiping algorithms. From a theoretical standpoint, it essentially relies on a special form of multiple time scale stochastic approximation. The resulting technique may have other applications besides user profiling. 

Finally we evaluated our proposed methods on synthetic and actual data traces. We thereby validated our analysis in observing convergence to the desired eigenvectors. We could further show, based on the Netflix prize dataset, that accurate recommendations could be made at limited communication cost based on our spectral embedding. 

Several research directions can be envisioned to take this work further. One intriguing problem concerns privacy. While our methods do not rely on direct exchange of sensitive private information, they may nevertheless lead to private information leakage. A distributed solution avoiding the issue is yet unavailable.

Other directions concern the fine tuning of the methods. The issue of analytical selection of the number of eigenvectors has not been addressed here. The recent work of Shi et al. \cite{eigenspaces} could be an appealing solution.

\bibliographystyle{plain}
\bibliography{refs}

\appendix
\section {Proof of Lemma~\lowercase{\ref{lemma:dominant}}}
Consider an eigenvalue $\bar \lambda\ne 0$ of $\bar A$ and a corresponding normalised eigenvector $\bar x$, $\|\bar x\|_2=1$. For $1 \le u \le N$,
\begin{equation*}
(\bar A \bar x)(u) = \frac \omega N \sum_{v=1}^N b_{k(u)k(v)}\bar x(v) 
= \frac \omega N \sum_{\ell=1}^K b_{k(u)\ell} \sum_{v : k(v) = \ell} \bar x(v) = \bar \lambda \bar x(u).
\end{equation*}
For a large enough $N$, we have that $|C_k| > 1$, for all $k$. This is true, since the size of each class grows linearly with $N$. Then for all $k$ and for all $u, u' \in C_k$, with $u' \ne u$ it follows that $\bar x(u') = \bar x(u)$.

Denote the value of $\bar x(u)$ by $\hat y(k(u))$.
Then, for all $1 \le \ell \le K$, 
\begin{equation*}
\sum_{\ell'=1}^K \alpha_{\ell'} b_{\ell\ell'} \hat y(\ell') = \frac {\bar \lambda} \omega \hat y(\ell).
\end{equation*}

Thus $\hat y$ is an eigenvector of the $K \times K$ matrix \mbox{$M = B\diag\alpha$} corresponding to its eigenvalue~\mbox{$\frac {\bar \lambda} \omega$}. Since $M$ is a constant matrix, its eigenvalues are also constants. Hence it must be that there exists a constant $c$ such that $\bar \lambda = c \omega = \Theta( \omega )$. By Condition $\eqref{C2}$ we conclude that the top $L$ magnitude eigenvalues of $\bar A$ have distinct absolute values.

Finally, since $\bar x(u) = \hat y(k(u))$ we have that 
$$
1 = \|\bar x\|_2 = \sqrt N \|\hat y\|_{\alpha} \iff \|\hat y\|_{\alpha} = \frac 1 {\sqrt N}.
$$
Since $y = \frac {\hat y} {\|\hat y\|_{\alpha}}$, we must have that $x(u) = \hat y(k(u)) =  \frac {y(k(u))} {\sqrt N}$, which proves \eqref{eq:eigcorresp}.
\hfill$\Box$

Note that a consequence of this result is that $\bar A$ can have a number of non-zero eigenvalues that is at least $L$ (by Condition \eqref{C1}) and at most $K$ (i.e. the maximum number of eigenvalues of matrix $M$).

\section{Proof of Lemma~\lowercase{\ref{lemma:feige}}}
\label{sec:proof-feige}
We start by establishing a simple result of stochastic dominance.

By definition a random variable $X$ is dominated for the convex ordering by another random variable $Y$ (written as $X \le_{cx} Y$) if for any convex function $f:[c,d] \to \mathbbm R$ such that $\mathbbm Ef(X)$ and $\mathbbm Ef(Y)$ exist, we have $\mathbbm Ef(X) \le \mathbbm Ef(Y)$.

It is known that (see for instance~\cite{stoyan}, Theorem~1.5.20): If $X \le_{cx} Y$, then there exist $\hat X$ and $\hat Y$ with the same distributions as the original variables, but which are such that $\mathbbm E (\hat Y|\hat X) = \hat X$.

Another result found in~\cite{stoyan} states that for some closed interval $[c,d]$, if $X:\Omega \to [c,d]$ and $Y:\Omega \to \{c,d\}$ are two random variables such that $\mathbbm EX = \mathbbm EY$, then $X \le_{cx} Y$.

In this latter setting, we wish to establish a variant of the former result. Namely, for some closed interval $[c,d]$, if $X:\Omega \to [c,d]$, we wish to construct a random variable $\hat Y:\Omega \to \{c,d\}$ supported on the extremities $\{c, d\}$ of the interval, such that $\mathbbm EX = \mathbbm E\hat Y$ and $\mathbbm E (\hat Y|X) = X$. To achieve this, pick a uniformly distributed random variable $U \sim \mathcal U [0,1]$ independent of $X$. We define the random variable $\hat Y:\Omega \to \{c,d\}$ as $\hat Y = F( X, U )$, where
$$
F(x,u) = d - \mathbbm 1_{\{x \le u (d-c) + c\}} ( d - c ).
$$
Then 
$$
\mathbbm E (\hat Y|X) = \mathbbm E ( F( X, U ) |X) = 
d - \frac {d-X}{d-c} (d-c) = X.
$$
Since $\hat Y$ can possibly take only two values, we can compute the corresponding probabilities:
\begin{equation}
\label{eq:prb}
\mathbb P( \hat Y = d ) = \int_c^d P_X(dx) \int_0^{\frac {x-c}{d-c}} du = \frac {\mathbbm E X - c}{d-c} \mbox{ and } \mathbb P( \hat Y = c ) = \frac {d - \mathbbm E X}{d-c},
\end{equation}
and hence $\mathbbm E X = \mathbbm E \hat Y$.

Let us now proceed with the proof of the Lemma.

Denote by $Q := A - \mathbbm E A$. Denote further by $p_{max} = \max_{ij} p_{ij}$ and by $p_{min} = \min_{ij} p_{ij}$. Then the elements of $Q$ all belong to the interval $[-p_{max}, 1 - p_{min}]$. For a large enough $N$ such that $p_{max} < 1$, there exist $\alpha = \frac {1 - p_{min}} {1-p_{max}} > 1$ and $p = p_{max}$, such that $[-p_{max}, 1 - p_{min}] \subset \alpha [-p, 1-p]$.

Consider a symmetric matrix $U$ of independent uniformly distributed random variables $\{U_{ij} = U_{ji} \sim \mathcal U[0,1]\}_{i<j}$. Then for each entry $Q_{ij}$ of $Q$, such that $i<j$, which we regard as a random variable with values in the interval $\alpha[-p,1-p]$, we construct the random variable $Z_{ij} = F( Q_{ij}, U_{ij} )$ which has the desired property written in matrix form
\begin{equation}
\label{eq:cexp}
\mathbbm E (Z|Q) = Q.
\end{equation}
Note that the entries of $Z$ are mutually independent by construction. Furthermore, the random variables defined as $\{Y_{ij} = \alpha^{-1} Z_{ij} + p, \ Y_{ji} = Y_{ij}\}_{i<j}$ are mutually independent Bernoulli random variables of parameter $p$ and hence form the adjacency matrix $Y$ of an Erdos-Renyi graph of parameters $(N,p)$. Denote $\bar Y := \mathbbm E Y = p(ee' - I)$, where $e$ is the all-ones column vector and $e'$ denotes transposition. 

Let us now prove that the spectral radius of $Z = \alpha(Y - \bar Y)$ is upper bounded by $O(\sqrt \omega)$ with high probability.

Since $\omega = \Omega( \log N )$, and the $\{Y_{ij}\}_{i<j}$ are mutually independent, we can apply the results from~\cite{feige}. Let $y$ be any vector of norm $1$ and denote $u := \frac  1 {\sqrt N} e$. We can decompose $y$ as follows: $y = ax + bu$, where $a^2 + b^2 = 1$, and $x$ is a vector of norm $1$ orthogonal to $u$, $x \perp u$.
\begin{equation*}
|y'(Y - \bar Y)y| \le 2\underbrace{|ab x'(Y-\bar Y)u|}_{T_1} + a^2\underbrace{|x'( Y - \bar Y )x|}_{T_2}
+ b^2\underbrace{|u'( Y - \bar Y ) u|}_{T_3}.
\end{equation*}
Denote by $\delta_i$ the degree of node $i$ and by $\bar \delta := \frac {\sum_i \delta_i} N$ the average degree. According to Lemma 2.2 from~\cite{feige} and taking into account the fact that $e$ is an eigenvector of $\bar Y$ we have
\begin{equation*}
T_1 = |x'Yu| \le 2 \sqrt{ \bar \delta },\mbox{ with probability $1 - e^{-\Omega((N\omega)^{1/3})}$}.
\end{equation*}

By Theorem 2.5 and Claim 2.4 from~\cite{feige}, we have that for every constant $c_1 > 0$, there exists another constant $c_2 > 0$ such that:
\begin{equation}
\label{eq:fo25}
|x'Yx| \le c_2 \sqrt \omega,\mbox{ with probability $1 - N^{-c_1}$.}
\end{equation}
We will restrict ourselves to constants $c_1 > 1$ for reasons that will become apparent later in the proof. Thus, we can bound the second term with probability $1 - N^{-c_1}$:
\begin{eqnarray*}
T_2 & \le & |x'Yx| + |x'\bar Yx| \stackrel{\mbox{\scriptsize(\ref{eq:fo25})}}{\le} 
O(\sqrt \omega) + |\sum_i x_i \sum_{j:i \ne j} p x_j| \\
& = & O(\sqrt \omega) + |- p \sum_i x_i^2 | = O( \sqrt \omega ) + \Theta (\frac \omega N).
\end{eqnarray*}
Finally, using a Chernoff bound we find
\begin{equation*}
T_3 = |\bar \delta - \omega| = O( \sqrt \omega ),\mbox{ with prob. $1 - e^{ -\Omega( N ) }$}.
\end{equation*}

Thus, $\rho( Z ) = O( \sqrt \omega )$, with probability $1 - N^{-c_1}$. Furthermore, there exists a constant $a>0$ such that $\rho( Z ) < aN$.

Let us now finally characterise the spectral radius of $Q$.

Using the fact that the spectral radius is a convex function and by Jensen's inequality, we get
$$
\rho( Q ) = \rho( \EE{ Z | Q } ) \le \EE{ \rho( Z ) | Q } 
$$
We have a random variable $R := \rho( Z )$ supported on $[0,aN]$, such that $\mathbbm P(R>t)\le O(N^{-c_1})$ for $t=O(\sqrt{\omega})$ and we wish to deduce that the conditional expectation $S:=\mathbbm E(R | Q)$ is also upper bounded by $O(\sqrt \omega)$ with high probability.

Since $R$ and $S$ have countable state spaces, it makes sense to consider
$$
\beta(s) := \mathbbm P ( R>t | S=s ).
$$
Since $R : \Omega \to [0, aN]$, and since $\mathbbm E( R | S ) = S$,
we have that
$
\beta(s)(aN)+(1-\beta(s))t \ge s,
$
and thus, 
$
\beta(s)\ge (s-t)/(aN-t).
$
Denoting $\gamma :=\mathbbm P(S> t+1)$, we have 
$$
\mathbbm P(R>t)=E(\beta(S))\ge E(\beta(S) \mathbbm 1_{\{S>t+1\}} ) \ge \frac {t+1-t} {aN-t} \gamma = \frac \gamma {aN-t}.
$$
Hence 
$$
\gamma=\mathbbm P(S>t+1)\le (aN-t) \mathbbm P(R>t)= (aN-t) O(N^{-c_1})=o(1),
$$
since we considered $c_1>1$.
\hfill$\Box$

\section {Proof of Lemma~\lowercase{\ref{lemma:pert}}}
We will show the two claims~(\ref{eq:eigdiff}) and~(\ref{eq:vecdiff}) by induction. Denote~(\ref{eq:eigdiff}) by $\mathcal P_k$ and~(\ref{eq:vecdiff}) by $\mathcal Q_k$.

We will begin by proving $\mathcal P_1$ and $\mathcal Q_1$. Since we make extensive use of Lemma~\ref{lemma:feige}, it is implied that all inequalities in this proof hold with high probability in the sense of Lemma~\ref{lemma:feige} (that is with probability $1 - e^{-\Omega(N)}$). All vectors are column vectors and we denote by $x'$ the transpose of vector $x$. Furthermore, for two vectors $x$ and $y$ by $x \perp y$ we mean that their scalar product $x'y$ equals zero.

Using the variational characterisation of eigenvalues and Lemma~\ref{lemma:feige}, we get:
\begin{equation*}
\left| |\lambda_1| - |\bar \lambda_1| \right| \stackrel{\mbox{\scriptsize wlog}}{=} |\lambda_1| - |\bar \lambda_1| 
\le |\lambda_1| - |x_1'\bar A x_1| \le | x_1' (A - \bar A) x_1 | \le O( \sqrt \omega ),
\end{equation*}
which proves $\mathcal P_1$.

We denote the first eigenvector of $A$ by $x_1 = a_1 \bar x_1 + b_1 \bar y_1$, where $a_1^2 + b_1^2 = 1$, $\bar x_1$ is the first eigenvector of $\bar A$ and $\bar x_1 \perp \bar y_1$. Then, by making use of $\mathcal P_1$, there exist positive constants $\theta_1$ and $\theta_2$ such that
\begin{equation*}
|\bar \lambda_1| - \theta_1 \sqrt \omega \le |\lambda_1| \le a_1^2|\bar x_1'\bar A \bar x_1| + b_1^2|\bar y_1'\bar A \bar y_1| + \theta_2 \sqrt \omega.
\end{equation*}
We took into account the symmetry of $\bar A$ and the fact that $\bar A \bar x_1 = \bar \lambda_1 \bar x_1$. By the Courant-Fischer theorem, we get the following inequality:
\begin{equation}
\label{eq:pert}
|\bar \lambda_1| - \theta \sqrt \omega \le |\bar \lambda_1 | - b_1^2( |\bar \lambda_1| - |\bar \lambda_2| ), \ \theta > 0
\end{equation}
and since the top $L$ eigenvalues of $\bar A$ are distinct (by Condition~\eqref{C2}), it holds that $0 < |\bar \lambda_1| - |\bar \lambda_2| = \Theta( \omega )$. We get that \mbox{$b_1^2 \le O(\frac 1 {\sqrt \omega})$}, thus proving $\mathcal Q_1$.

In order to generalise this result, we make use of the following simple lemma:
\begin{lemma}
\label{lemma:perp}
Let $x_1$ and $\bar x_1$ be two non-orthogonal vectors of norm $1$ such that
\begin{equation}
\label{eq:obv}
1 - (x_1' \bar x_1)^2 \le O(\frac 1 {\sqrt \omega}).
\end{equation}
Then for all vectors $x_2$ of norm $1$ such that $x_2 \perp x_1$, $(x_2' \bar x_1)^2 \le O(\frac 1 {\sqrt \omega})$.
\end{lemma}
\begin{proof}
We have $x_1 = a\bar x_1 + b\bar y_1$, where $\bar x_1 \perp \bar y_1$, $\| \bar y_1\| = 1$ and $a^2 + b^2 = 1$. By hypothesis~(\ref{eq:obv}), $b^2 \le O( \frac 1 {\sqrt \omega})$. Thus,
\begin{equation*}
x_2' x_1 = a x_2' \bar x_1 + b x_2' \bar y_1 = 0, 
\end{equation*}
and thus, since $a \ne 0$,  
\begin{equation*}
|x_2' \bar x_1|^2 = \frac {b^2} {1 - b^2} |x_2' \bar y_1|^2 \le \theta \frac 1 {\sqrt \omega} \frac {\sqrt \omega} {\sqrt \omega - \theta}, 
\end{equation*}
where $\theta > 0$ is a constant.
\end{proof}

We proceed by complete induction. We showed $\mathcal P_1$ and $\mathcal Q_1$. Now assume $\mathcal P_\ell$ and $\mathcal Q_\ell$ are true for all $1 \le \ell < k$. We wish to show $\mathcal P_k$ and $\mathcal Q_k$.

Let us write $x_k = \alpha \bar x_k + \beta \bar y + \gamma \bar z$, where $\bar y \in \mbox{Span}\{\bar x_{k+1}, \bar x_{k+2}, \dots\}$ and $\bar z \in \mbox{Span}\{\bar x_1, \dots, \bar x_{k-1}\}$ and $\alpha^2 + \beta^2 + \gamma^2 = 1$. Lemma~\ref{lemma:perp} and the induction hypothesis show that $\gamma^2 \le O( \frac 1 {\sqrt\omega} )$.
Since
\begin{equation*}
|x_k'\bar A x_k| \le ( 1 - \gamma^2 ) |\bar \lambda_k| + \gamma^2 |\bar \lambda_1| \le |\bar \lambda_k| + O(\sqrt \omega),
\end{equation*}
we can conclude that
\begin{equation*}
\left| |\lambda_k| - |\bar \lambda_k| \right| \stackrel{\mbox{\scriptsize wlog}}{=} |\lambda_k| - |\bar \lambda_k|
\le |\lambda_k| - |x_k'\bar A x_k| + O( \sqrt \omega ) \le O( \sqrt \omega ),
\end{equation*}
thus proving $\mathcal P_k$.

We have that,
\begin{eqnarray*}
|\lambda_k| & \le & |x_k'\bar A x_k| + |x_k'(A - \bar A) x_k|
\le \alpha^2 |\bar \lambda_k| + \beta^2 |\bar \lambda_{k+1}| + \gamma^2 |\bar \lambda_1| + O( \sqrt \omega )\\
 & \le & |\bar \lambda_k| + \frac a {\sqrt \omega} ( |\bar \lambda_1| - |\bar \lambda_k| )
 - \beta^2 ( |\bar \lambda_k| - |\bar \lambda_{k+1}| ) + \theta \sqrt \omega,
\end{eqnarray*}
where $a$ and $\theta$ are positive constants. Without loss of generality we now assume that $\left| |\lambda_k| - |\bar \lambda_k| \right| = |\bar \lambda_k| - |\lambda_k|$ and using $\mathcal P_k$ we get
\begin{eqnarray*}
\beta^2 \le \frac {\varphi \sqrt \omega + \frac a {\sqrt \omega} ( |\bar \lambda_1| - |\bar \lambda_k| )} {|\bar \lambda_k| - |\bar \lambda_{k+1}|} \le O( \frac 1 {\sqrt \omega}), \ \varphi > 0
\end{eqnarray*}
thus proving $\mathcal Q_k$.
\hfill$\Box$

\section{Technical Lemmas for Proving Theorem~\ref{thm:rmain}}
\label{app:rectlemmas}
We make the following notation:
$$
\bar S := (pr_{k(u)k'(i)})_{u,i} = \mathbb E S.
$$
The following Lemmas characterise the structure of the singular decomposition of $\bar S$ and $S$. They show that the two matrices have the same spectral structure. 
\begin{lemma}
\label{lemma:rdominant}
For $L\le K'$, the top $L$ largest singular values of $\bar S$ are distinct and of order $\Theta(\omega)$. The normalised left-singular vectors $(\bar x_\ell)_{\ell = 1}^L$ corresponding to these singular values are constant on indices corresponding to each user class. Specifically, using the $g_\ell$ defined in \eqref{D3}, we can write
\begin{equation}
\label{eq:svcorresp}
\bar x_\ell(u) = \frac {g_\ell(k(u))} {\sqrt N}, \ \forall u \in \mathcal U, \ 1\le \ell \le L.
\end{equation}
\end{lemma}
\begin{proof}
Consider a non-zero singular value of $\bar S$, $\sigma > 0$ and corresponding left and normalised right singular vectors $\bar x$ and $\bar y$. Then we can write:
$$
(\bar S' \bar x)(i) = \sum_{u=1}^N \bar S_{ui} \bar x(u) = \sigma \bar y(i) \stackrel{\mbox{\scriptsize not.}}{=} \sigma h(k'(i)), \ 
(\bar S \bar y)(i) = \sum_{i=1}^F \bar S_{ui} \bar y(i) = \sigma \bar x(u) \stackrel{\mbox{\scriptsize not.}}{=} \sigma \hat g(k(u)),
$$
since $\bar S_{ui} = p r_{k(u)k'(i)}$ depends only on the class of $u$, $k(u)$, and the class of $i$, $k'(i)$. Since $\bar y$ is a right-singular vector of $\bar S$, it is also an eigenvector of $\bar S' \bar S$ corresponding to eigenvalue $\sigma^2$, and thus after some simplification we have that
$$
\frac {\sigma^2}{\gamma \omega^2} \hat g = R \diag \beta R' \diag \alpha \hat g,
$$
that is the $K$-dimensional vector $\hat g$ is an eigenvector of
the constant matrix $R \diag \beta R' \diag \alpha$ corresponding to
the eigenvalue $\frac{\sigma^2}{\gamma \omega^2}$. Since the
eigenvalues of a constant matrix are constant, it follows that there
exists a constant $\lambda$ such that $\sigma = \lambda \omega =
\Theta( \omega )$. By Condition~\eqref{D2} it follows that the largest
$L$ singular values of $\bar S$ are distinct. Moreover, matrix $R
\diag \beta R' \diag \alpha$ can have at most $K$ distinct non-zero
eigenvalues and hence the same holds for the singular values of $\bar
S$.

Finally, we have that 
$$
1 = \|\bar x\|_2 = \sqrt N \|\hat g\|_{\alpha}.
$$
Thus, if we consider a vector $g = \frac {\hat g} {\|\hat
  g\|_\alpha}$ normalised under the $\alpha$-norm, Equation~\eqref{eq:eigcorresp} holds.
\end{proof}

We denote the singular values of $\bar S$ and $S$ by
\begin{equation}
\label{eq:svs}
\begin{array}{l}
\bar \sigma_1 > \bar \sigma_2 > ... > \bar \sigma_L > 0  \\
\sigma_1 \ge \sigma_2 \ge ... \ge \sigma_L > 0.
\end{array}
\end{equation}
and the corresponding left and right normalised singular vectors by $\bar x_k$, $\bar y_{k'}$, $x_k$, and $y_{k'}$.
\begin{lemma}
\label{lemma:rpert}
For all $1 \le k \le K$
\begin{eqnarray}
\label{eq:reigdiff}
\left| \sigma_k - \bar \sigma_k \right| \le O(\sqrt \omega) & & \mbox{whp}, \\
\label{eq:rvecdiff}
\sin( \widehat{ x_k, \bar x_k } ) \le O( \omega ^ {-1/4} ) & & \mbox{whp}.
\end{eqnarray}
\end{lemma}
\begin{proof}
For $\bar A = \tau \bar S$ we denote by $\bar \zeta_k^+$ the
normalised eigenvector corresponding to the eigenvalue $\bar \sigma_k$
and by $\bar \zeta_k^-$ the normalised eigenvector corresponding to
the eigenvalue $-\bar \sigma_k$. We introduce similar notation for the
eigenvectors of $A = \tau S$, namely $\zeta_k^+$ and $\zeta_k^-$. Then
it holds that
$$
\bar \zeta_k^+ = \frac 1 {\sqrt 2} \left[\begin{array}{c} \bar x_k \\ \bar y_k\end{array}\right], \ 
\bar \zeta_k^- = \frac 1 {\sqrt 2} \left[\begin{array}{c} \bar x_k \\ -\bar y_k\end{array}\right], \ 
\zeta_k^+ = \frac 1 {\sqrt 2} \left[\begin{array}{c} x_k \\ y_k\end{array}\right], \ 
\zeta_k^- = \frac 1 {\sqrt 2} \left[\begin{array}{c} x_k \\ -y_k\end{array}\right].
$$

By Condition~\eqref{D2}, we can apply a slightly modified Lemma~\ref{lemma:pert} to matrix $A$. The only modification we need to make is to change the considered ordering of the eigenvalues -- instead ordering them by largest magnitude, we order them decreasingly by value. Since we can apply Lemma~\ref{lemma:feige} to $A-\bar A$, it is straightforward to see that the proof also holds in this setting.

We have thus that~(\ref{eq:reigdiff}) holds and that with high probability
$$
\sin( \widehat{ \zeta_k^+, \bar \zeta_k^+ } ) \le O( \omega ^ {-1/4} ).
$$
To see that~(\ref{eq:rvecdiff}) holds as well, we write
\begin{eqnarray*}
1 - \DP{\zeta_k^+}{\bar \zeta_k^+}^2  = 
1 - \frac 1 4 ( \DP{x_k}{\bar x_k} + \DP{y_k}{\bar y_k} )^2 & \le & O( \omega ^ {-1/2} ), \\
- \DP{\zeta_k^+}{\bar \zeta_k^-}^2 = 
- \frac 1 4 ( \DP{x_k}{\bar x_k} - \DP{y_k}{\bar y_k} )^2 & \le & 0.
\end{eqnarray*}
By summing the two expressions we get that: $ \frac 1 2 (1 -
\DP{x_k}{\bar x_k}^2) + \frac 1 2 (1 - \DP{y_k}{\bar y_k}^2) \le O(
\omega ^ {-1/2} ).  $
\end{proof}

\section{Applicability of Theorem~\lowercase{\ref{thm:main}} in the setting of Section~\ref{sec:drw}}
\label{app:port}
In Section~\ref{sec:sq} we had defined the profile of user $u$ as
$\sqrt N z_u'$, a scaled row vector containing the \mbox{$u$-th}
coordinate of each of the top $L$ eigenvectors of matrix $A$ (i.e.,
the user profiles are scaled rows of the $N \times L$ matrix $X=(x_1,
\dots, x_L)$ of the top $L$ normalised eigenvectors of $A$). The
algorithms we investigate in Section~\ref{sec:drw} produce slightly
different coordinates for the users: these coordinates form a
collection of $L$ linearly independent vectors which span the vector
space generated by the $L$ eigenvectors corresponding to the top $L$
largest magnitude eigenvalues of $A$.

Specifically, for some unknown full rank $L\times L$ matrix $W = (w_1,
\dots, w_L)$ of linear coefficients, we redefine the profile of user
$u$ as a scaled row vector $\sqrt N \hat z'_u$ containing the $u$-th
coordinates of an orthonormal basis of the space spanned by the top
$L$ eigenvectors (i.e., the rows of the matrix $XW$): $ \hat z_u =
((Xw_1)(u), \dots, (Xw_L)(u))$. Even in this setting,
Theorem~\ref{thm:main} and its corrolaries still apply under the same
assumptions~\eqref{C14}. Hence, for a large number of users, clusters
will emerge, and there is no need to know the matrix of linear
coefficients $W$.

To give an intuition as to why this result still holds, recall
Lemma~\ref{lemma:dominant} which shows that the top $L$ largest
eigenvectors $\bar X$ of the block matrix $\bar A$ are constant on
indices within the same user class. Hence, a linear combination
thereof has the same property. Thus, the redefined user profiles
concentrate around constant vectors $(\hat t'_k)_k$ corresponding to
the user classes: $\hat t_k':=((Yw_1)(k),\ldots, (Yw_L)(k))$, where $Y
= (y_1,\ldots,y_{L})$ is the matrix of eigenvectors of $M = B \diag
\alpha$ normalised under the $\alpha$-norm. As previously stated, the
following condition needs to hold in order to distinguish users of
separate classes:
\begin{align*}
\parbox[b]{.85\textwidth}{The normalized eigenvectors $Y$ under the $\alpha$-norm are such that $\hat t_k'\ne \hat t_\ell', \ k\ne \ell$. } & \mbox{(\ref{C3}$'$)}.
\end{align*}
Due to the fact that matrix $W$ is full rank, one can show that
condition~(\ref{C3}$'$) is equivalent to~\eqref{C3}. Moreover, the
averaging argument in the proof of the theorem yields:
\begin{align*}
\frac 1 N \left|\left\{ u: \|\hat z_u'-\hat t_{k(u)}'\|\ge a\right\}\right|&\le\frac{1}{N}\sum_{u=1}^N\frac{\|\sqrt{N} \hat z_u'-\hat t_{k(u)}'\|^2}{a^2}
= a^{-2} \sum_{u = 1} ^ N \sum_{\ell = 1} ^ L \left[ \sum_{i = 1} ^ L w_\ell(i) [x_i(u) - \bar x_i(u)] \right]^2 \\
& \mathop{\le}^{\text{Cauchy-Schwartz}} \|W\|_F^2 \sum_{\ell = 1} ^ L \frac {\|x_\ell - \bar x_\ell\|^2}{a^2} = O(a^{-2}\omega^{-1/2}),
\end{align*}
since the Frobenius norm $\|W\|_F$ of matrix $W$ is constant (all of its elements are subunitary).

\section{Proof of Theorem~\lowercase{\ref{thm11}}}
\label{proof:thm11}

The main ingredient in the proof is the following
\begin{lemma}
  \label{lemme_dist1} 
  The update equation (\ref{dist1}) is such that for all $t>0$, 
  \begin{equation}\label{eq_lemme_dist1} 
    \|X(t+1)\| \le e^{K\sum_{s=1}^{t}a(s)}\left(\|X(1)\|+M\right)
  \end{equation}
  for some positive constant $K$ (i.e., that does not depend on $t$)
  and some almost surely finite random variable $M$.
\end{lemma} 
\begin{proof}
  Rewrite Equation~(\ref{dist1}) in matrix form as 
  $$
  X(t+1)-X(t)=a(t)\left[F((X,\Phi,\Psi)(t))+ D^{-1}(t)\xi(t+1)\right]
  $$
  for some suitable function $F$, and where $D(t)$ denotes the $N\times N$ diagonal matrix with diagonal entries $Y_u(t)$. By the specific form of the terms $Y_u(t)$, and their role in the function $F$, it is readily seen that the latter verifies
  $$
  \|F(X,\Phi,\Psi)\|\le K_1 \| X\|
  $$
  for some suitable constant $K_1$. This readily implies that
  $$
  \|X(t+1)\|\le(1+ K_1 a(t))\|X(t)\|+a(t)\eta(t),
  $$
  where the $\eta(t)= \|\xi(t+1)\|$ are iid. By induction, one then establishes that
  $$
  \|X(t+1)\|\le \prod_{s=1}^{t}(1+K_1 a(s))\left[\|X(1)\| + \sum_{s=1}^t a(s)\eta(s)\right].
  $$
  Denote now $\bar{\eta}$ the expectation of $\eta(s)$, and let $M(t):=\sum_{s=1}^t a(s)[\eta(s)-\bar{\eta}]$. 
  It is readily seen that $M(t)$ is a uniformly integrable martingale, and hence the supremum $\sup_{s>0}M(s)$ is almost surely finite; denote it by $\hat{M}$. It then follows from the above equation that
  $$
  \|X(t+1)\|\le \prod_{s=1}^{t}(1+K_1 a(s))\left[\|X(1)\| + \sum_{s=1}^t \bar{\eta}a(s)+\hat{M}\right].
  $$
  Using the elementary inequality $1+x\le e^x$, one deduces:
  $$
  \|X(t+1)\|\le e^{\sum_{s=1}^t K_1 a(s)}\left[\|X(1)\|+e^{\sum_{s=1}^t\bar{\eta}a(s)}+\hat{M}\right].
  $$
  The result~(\ref{eq_lemme_dist1}) then follows by setting $K=K_1+\bar{\eta}$, and $M=1+\max(0,\hat{M})$.
\end{proof}

An additional result that is used is the following
\begin{lemma}
  \label{lemma_cv0}
  For a given sequence $\epsilon(t)$ with
  $\lim_{t\to\infty}\epsilon(t) = 0$, the sequence $z(t)$ defined as
  $$
  z(t+1) = (1-\lambda b(t))z(t) + b(t)\epsilon(t+1)
  $$
  converges to $0$ as $t$ goes to $\infty$.
\end{lemma}
\begin{proof}
  By induction, one can deduce from the previous expression the identity
  $$
  z(t+1)=z(1)\prod_{s=1}^t(1-\lambda b(s))+\sum_{s=1}^t \epsilon(s) b(s) \prod_{\sigma=s+1}^t (1 - \lambda b(\sigma)). 
  $$
  Elementary analysis can then be used to deduce from this last display, assumptions~\eqref{ii},\eqref{iii} on gains $b(t)$, and convergence of $\epsilon(t)$ to 0 that $z(t)$ also converges to 0. 

  Indeed, take any fixed $\varepsilon > 0$. Since $\epsilon(t) \rightarrow 0$, there exists some $t_0$ such that $\epsilon(t) < \frac {\varepsilon \lambda} 3$, for all $t \ge t_0$. Then, we can write:
  \begin{eqnarray}
    \label{eq:lpart1}
    z(t+1) & \le & z(1)\prod_{s=1}^t(1-\lambda b(s)) \\
    \label{eq:lpart2}                          
    & &  + \sum_{s=1}^{t_0} \epsilon(s) b(s) \prod_{\sigma=s+1}^t (1 - \lambda b(\sigma)) \\
    \label{eq:lpart3}
    & & + \frac {\varepsilon \lambda} 3 \sum_{s=1}^t b(s) \prod_{\sigma=s+1}^t (1 - \lambda b(\sigma)).
  \end{eqnarray}
  Term~(\ref{eq:lpart1}) develops as
  $$
  z(1)\prod_{s=1}^t(1-\lambda b(s)) = z(1)e^{\sum_{s=1}^t \log (1-\lambda b(s))} \le z(1)e^{-\lambda\sum_{s=1}^t b(s)} \rightarrow 0,
  $$
  by assumption~\eqref{iii}. Term~(\ref{eq:lpart2}) consists of $t_0$ terms (a finite number), each of which converges to 0 by the same argument as above. Hence, there exists $t_1$ such that for all $t\ge t_1$, we have 
  $$
  z(1)\prod_{s=1}^t(1-\lambda b(s)) + \sum_{s=1}^{t_0} \epsilon(s) b(s) \prod_{\sigma=s+1}^t (1 - \lambda b(\sigma)) \le \frac {2\varepsilon} 3.
  $$

  It can be shown that term~(\ref{eq:lpart3}) can be written as:
  $$
  \frac {\varepsilon \lambda} 3 \sum_{s=1}^t b(s) \prod_{\sigma=s+1}^t (1 - \lambda b(\sigma)) = \frac {\varepsilon \lambda} 3 \frac 1 \lambda \left(1 - \prod_{s=1}^t(1-\lambda b(s))\right) \le \frac \varepsilon 3.
  $$
  We have shown that for all $\varepsilon > 0$, there exists some $t_2 = t_0 \vee t_1$ such that for all $t>t_2$, we have $z(t+1) \le \varepsilon$.
\end{proof}

The previous lemmas are now used to establish the following result:
\begin{lemma}
  The auxiliary variables $\Phi_u(t)$, $\Psi_u(t)$ verify
  \begin{equation}
    \begin{array}{l}
      \lim_{t\to\infty}|N \Phi_u(t)-\sum_v f_v(t)|=0,\\
      \lim_{t\to\infty}|N\Psi_u(t)-\sum_v g_v(t)|=0,
    \end{array}
  \end{equation}
  i.e. asymptotically, these quantities do track accurately their intended targets.
\end{lemma}
\begin{proof}
  We shall only consider the case of $\Psi_u(t)$, the other one being entirely similar. Rewrite the update rule (\ref{dist3}) in matrix form as
  \begin{equation}\label{eq_lemme_dist2}
    \Psi(t+1)=(I-b(t)\Lambda)\Psi(t)+g(t+1)-g(t),
  \end{equation}
  where $\Lambda$ is the so-called Laplacian matrix of the overlay graph: $\Lambda_{uu}=|{\mathcal N}_u|$, $\Lambda_{uv}=-\mathbbm 1_{u\sim v}$ for $u\ne v$. Recall that the Laplacian $\Lambda$ is positive semi-definite, with eigenvector $e=(1,\ldots,1)'$ associated to the eigenvalue 0. Also, when the overlay graph is connected, all other eigenvectors are associated with strictly positive eigenvalues $\lambda>0$. 

  We have by definition of $g$:
  \begin{align*}
    g(t+1) - g(t) &= \{X_u(t+1)X_u(t+1)' - X_u(t)X_u(t)'\}_u \\
    &=  \{X_u(t+1)(X_u(t+1)' - X_u(t)')\}_u +  \{(X_u(t+1) - X_u(t)) X_u(t)')\}_u \\
    &= a(t) \{X_u(t+1)[F_u((X,\Phi,\Psi)(t))']\}_u +  \{[F_u((X,\Phi,\Psi)(t))] X_u(t)')\}_u \\
    &+ a(t) \{X_u(t+1) \frac {\xi_u(t+1)'} {Y_u(t)} + \frac {\xi_u(t+1)} {Y_u(t)} X_u(t)'\}_u.
  \end{align*}

  We can rewrite this as
  $
  g(t+1) - g(t) = a(t) (r(t+1)+s(t+1)+w(t+1)),
  $
  where
  \begin{align*}
    r(t+1) &= \{(X_u(t+1) + X_u(t))[F_u((X,\Phi,\Psi)(t))']\}_u, \\
    s(t+1) &= \{[2X_u(t) + a(t)F_u((X,\Phi,\Psi)(t))]\frac {\xi_u(t+1)'} {Y_u(t)}\}_u, \\
    w(t+1) &= a(t) \{ \frac {\|\xi_u(t+1)\|^2} {Y_u^2(t)} \}_u
  \end{align*}

  By Lemma~\ref{lemme_dist1} and the bound on function $F$ therein we have that 
  $$
  \|r(t+1)\| \le K' e^{2K\sum_{s=1}^t a(s)}.
  $$
  Additionally,
  $$
  \mathds E[ s(t+1) | {\cal F}_t ] = 0.
  $$

  Consider $z$ an eigenvector of $\Lambda$ corresponding to a non-zero
  eigenvalue $\lambda > 0$ and define
  \begin{align*}
    \epsilon(t+1) &:= \frac {a(t)} {b(t)} z'[r(t+1)+w(t+1)], \\
    \epsilon'(t+1) &:= \frac {a(t)} {b(t)} z's(t+1).
  \end{align*}
  Then $\lim_{t\to\infty}\epsilon(t) = 0$ almost surely. For the term
  $\frac {a(t)} {b(t)} z'r(t+1)$ the convergence follows from
  condition~\eqref{iv}, while for $\frac {a(t)} {b(t)} z'w(t+1)$ it
  follows from the fact that the $\|\xi_u\|$ have finite variance and
  from condition~\eqref{iii}.

  Denote $\hat z(t)$ the scalar product $z' \Psi(t)$. One then deduces
  from (\ref{eq_lemme_dist2}) the equation
  $$
  \hat z(t+1)=(1-\lambda b(t))\hat z(t)+b(t)( \epsilon(t+1) + \epsilon'(t+1) ).
  $$
  We wish to show that $\hat z(t)$ converges to $0$.  We know that
  the sequence $z(t)$ defined in Lemma~\ref{lemma_cv0} converges to $0$.  
  Consider $\Delta(t) := \hat z(t) - z(t)$. It verifies
  $$
  \Delta(t+1) = (1-\lambda b(t))\Delta(t) + b(t)\epsilon'(t+1),
  $$ 
  with $\epsilon'$ such that $\mathds E[ \epsilon'(t+1) | {\cal F}_t ] = 0$.

  If we manage to show that $\Delta(t)$ converges to $0$, then the
  same follows for $\hat z(t)$. The convergence of $\Delta(t)$ follows
  from Theorem 1.IV.26. (Robbins-Monro) of~\cite{duflo}. We need to
  check that the hypothesis of the latter are verified. We only give
  details for the following condition $\sum_t b(t)^2 \mathds
  E[\|\epsilon'(t+1)\|^2] \le +\infty$, as the others are immediately
  verified.

  A sufficient condition is:
  $$
  \sum_t b(t)^2 \mathds E \left[\underbrace{\left(\frac {a(t)^2}
      {b(t)^2} e^{2K\sum_{s=1}^t a(s)}\right)}_{\to 0}(K''(1+\hat
    M))^2\|\xi(t+1)\|^2\right] < \infty,
  $$ where $\hat M$ is the supremum of the martingale $M(t) =
  \sum_{s=0}^t a(s) (\eta(s)-\bar \eta)$ from Lemma~\ref{lemme_dist1},
  with $\eta(s) = \|\xi(t+1)\|$, and $\bar \eta = \mathds E
  \eta(t+1)$.

  We need that $\mathds E(\hat M^2\|\xi(t+1)\|^2)<\infty$ and that
  $\sum_tb(t)^2 <\infty$ to conclude.

  $$ 
  \mathds E(\hat M^2\|\xi(t+1)\|^2) \le \frac 1 2 [\mathds E \hat
  M^4 + \underbrace{\mathds E \|\xi(t+1)\|^4}_{<\infty}].
  $$

  By Doob's inequality
  $$
  \mathds P(\sup_{s\le t}M(s) \ge x) \le \frac {\mathds E M(t)} {x}.
  $$
  Then
  \begin{align*}
    \mathds P(\sup_{s\le t}M(s) \ge x) = \mathds P(\sup_{s\le t}e^{\theta M(s)}
    \ge e^{\theta x}) \le  e^{-\theta x} \mathds E [e^{\theta M(t) }] = e^{-\theta x} e^{\sum_{s=0}^t \varphi(\theta a(s))},
  \end{align*}
  where $\varphi(y) := \log \mathds E e^{y(\eta-\bar\eta)}$ (by
  independence). Moreover, $\varphi(y) \approx C y^2$, for a small
  enough $y$ and some constant $C>0$.  Thus, there exists a large
  enough $s^*$ such that $\varphi(\theta a(s)) \le C (\theta a(s))^2$
  for all $s \ge s^*$
  \begin{align*}
    \mathds P(\sup_{s\ge 0}M(s) \ge x) \le e^{-\theta x}
    e^{\sum_{s=0}^{s^*} \varphi(\theta a(s)) + \sum_{s>s^*} C (\theta a(s))^2} = C'e^{-\theta x},
  \end{align*}
  since $\sum_s a(s)^2 < \infty$. Hence,
  \begin{align*}
    \mathds E \hat M^4 = \mathds E(\sup_{s\ge 0} M(s)^4) =
    \int_{0}^\infty \mathds P(\hat M^4 > t ) \ dt \le A
    \int_{0}^\infty e^{-\theta t^{1/4}} = \frac {24A} {\theta^4} <
    \infty,
  \end{align*}
  for some constant $A \ge 0$.

  We thus obtain that when decomposing vector $\Psi(t)$ according to
  the eigenbasis of matrix $\Lambda$, one finds vanishing coordinates
  except along eigenvector $e$. Since the scalar product $e'\Psi(t)$
  is always equal to $\sum_u g_u(t)$, by (\ref{eq_lemme_dist2}) the
  announced result follows.
\end{proof}

To conclude the proof of the theorem, note now that, by the previous
lemma, and our specific choice of gain parameters $Y_u(t)$ in
(\ref{dist2}), for large enough $t$, Equation~(\ref{dist1}) reads in
vector form
\begin{multline}
  \label{perturbed} X(t+1)-X(t) =\frac{a(t)}{\max(1,\sum_{k,v}X_{vk}^2(t)+o(1))}[AX(t)
    -X(t)X'(t)AX(t) \\
    +o(\|X(t)\|)+\xi(t+1)].
\end{multline}
As is readily seen, this coincides with the update rule
(\ref{borkar-meyn}), except for the $o(\cdot)$ terms. The analysis of
\cite{bmoja} for establishing convergence of (\ref{borkar-meyn}) also
applies in fact to its perturbed version (\ref{perturbed}), and
Theorem~\ref{thm11} follows.  \hfill$\Box$

\end{document}